\documentclass{article}
\usepackage{iclr2026_conference,times}
\usepackage{verbatim}
\usepackage{amsmath}
\usepackage{amsthm}
\usepackage{graphicx}
\usepackage{amssymb}
\usepackage{algorithm}
\usepackage{algpseudocode}
\usepackage{amsfonts} 
\usepackage{graphicx}
\usepackage{tabularx}
\usepackage{booktabs}
\usepackage{makecell}
\usepackage{url}

\usepackage[normalem]{ulem}

\usepackage[acronym, nonumberlist]{glossaries}
\newacronym{WAXp}{wAXp}{Weak Abductive Explanation}
\newacronym{distanceAXp}{$\epsilon$‑AXp}{distance-restricted explanations}
\newacronym{verixplus}{VERI{\large X}+}{}
\newacronym{verix}{VERI{\large X}}{}
\theoremstyle{definition}
\newtheorem{definition}{Definition}[section] 
\newtheorem*{definition*}{Definition}

\newtheorem{lemma}{Lemma}[section] 
\newtheorem*{lemma*}{Lemma}

\newtheorem*{theorem*}{Theorem}

\newtheorem{proposition}{Proposition}[section] 

\newtheorem*{problem*}{Problem}

\newcommand{\AXp}{\text{AXp}}

\newcommand{\wAXPA}{\text{wAXp}^A}
\newcommand{\wAXPAmin}{$\text{wAXp}^{A^{\star}}$}
\newcommand{\newtext}[1]{{\color{black}#1}}

\title{\textbf{FAME}: \uline{F}ormal \uline{A}bstract \uline{M}inimal \uline{E}xplanation for Neural Networks}

\author{
  Ryma Boumazouza\thanks{Equal contribution.}\ \   $^{1,2}$, Raya Elsaleh\footnotemark[1]\ \ $^{3}$, Melanie Ducoffe$^{1,2}$, Shahaf Bassan$^{3}$ and Guy Katz$^{3}$ \\
  \vspace{0.2cm}
  $^1$Airbus SAS, France, $^2$IRT Saint-Exupery, France, $^3$The Hebrew University of Jerusalem, Israel
}

\iclrfinalcopy
\begin{document}
\maketitle

\begin{abstract}
We propose \textbf{FAME} (Formal Abstract Minimal Explanations), a new class of abductive explanations grounded in abstract interpretation. FAME is the first method to scale to large neural networks while reducing explanation size. Our main contribution is the design of dedicated perturbation domains that eliminate the need for traversal order. FAME progressively shrinks these domains and leverages LiRPA-based bounds to discard irrelevant features, ultimately converging to a \textbf{formal abstract minimal explanation}.  
To assess explanation quality, we introduce a procedure that measures the worst-case distance between an abstract minimal explanation and a true minimal explanation. This procedure combines adversarial attacks with an optional \gls{verixplus} refinement step. We benchmark FAME against \gls{verixplus} and demonstrate consistent gains in both explanation size and runtime on medium- to large-scale neural networks. 
\end{abstract}
\section{Introduction}

\begin{figure}[h!]
    \centering
    \includegraphics[width=0.78\linewidth]{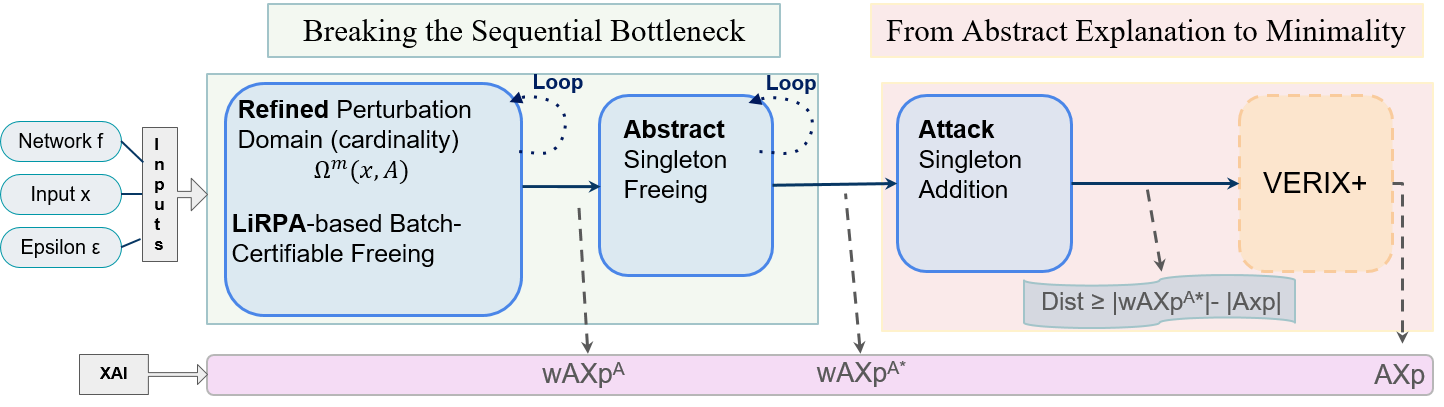}
    \caption{
    \textbf{FAME Framework.} 
    The pipeline operates in two main phases \newtext{\textbf{(1) Abstract Pruning (Green)}} 
    phase leverages abstract interpretation (LiRPA) to simultaneously free a large number of irrelevant \newtext{(pixels that are certified to have no influence on the model's decision)} features \newtext{based on a batch certificate} (Section~\ref{sec:abstract-simult-freeing}). \newtext{This iterative process operates} within a refined, cardinality-constrained perturbation domain, $\Omega^m(x,A)$ (Eq.~\ref{eq:refined-domain}) \newtext{to progressively tighten the domain}; To ensure that the final explanation is as small as possible, the remaining features that could not be freed in batches are tested individually (Section~ \ref {sec:perturbation_domain}). 
    \newtext{\textbf{(2) Exact Refinement (Orange)}} phase identifies the final necessary features using singleton addition attacks and, if needed, a final run of \gls{verixplus} (Section~\ref {sec:distance-to-minimality}). The difference in size, $\lvert \text{\wAXPAmin} \rvert - \lvert \AXp \rvert$, serves as an evaluation metric 
    of phase 1.
}
    \label{fig:fameframework}
\end{figure}

Neural network-based systems are being applied across a wide range of domains. Given AI tools' strong capabilities in complex analytical tasks, a significant portion of these applications now involves tasks that require reasoning. These tools often achieve impressive results in problems requiring intricate analysis to reach correct conclusions. Despite these successes, a critical challenge remains: understanding the reasoning behind neural network decisions. The internal logic of a neural network is often opaque, with its conclusions presented without accompanying justifications. This lack of transparency undermines the trustworthiness and reliability of neural networks, especially in high-stakes or regulated environments. Consequently, the need for interpretable and explainable AI (XAI) has become a growing focus in recent research.

Two main approaches have emerged to address this challenge. The first employs statistical and heuristic techniques to infer explanations based on network's internal representations \citep{fel2022xplique}. \newtext{While these methods estimate feature importance, that require empirical evaluation (such as the $\mu$-Fidelity metric \citep{ijcai2020p417}),
}the second approach leverages automated reasoners and formal verification to provide provably correct explanations grounded in logical reasoning.

\newtext{We ground our work in the formal definition of Abductive Explanations (AXp) \citep{ignatiev2019abduction}, a concept belonging to the broader family of "formal XAI" which includes 
}
minimal explanations, also known as local-minimal, minimal unsatisfiable subsets \citep{marques2010minimal} and prime implicants \citep{shih2018symbolic}. 
\newtext{An AXp is a subset of features guaranteed to maintain the model's prediction under any perturbation within a defined domain.} In a machine learning context, \newtext{these explanations characterize feature sets where removing any single feature invalidates the guarantee, effectively representing subsets that preserve the decision's robustness.} 
However, a major hurdle for formal XAI is its high computational cost due to the complexity of reasoning, preventing it from scaling to large neural networks (NNs) \citep{marques2023logic}. This limitation, combined with the scarcity of open-source libraries, significantly hinders its adoption. Initial hybrid approaches, such as the EVA method \citep{Fel_2023_CVPR}, have attempted to combine formal and statistical methods, but these often fail to preserve the mathematical properties of the explanation. However, robustness-based approaches address the scalability challenges of formal XAI for NN by leveraging a fundamental connection between AXps and adversarial examples \citep{huang2023robustness}.

In this work, we present FAME, a scalable framework for formal XAI that addresses the core limitations of existing methods. Our contributions are fourfold:

\begin{itemize}
\item \textbf{Formal abstract explanations.} We introduce the first class of abductive explanations derived from abstract interpretation, enabling explanation algorithms to handle high-dimensional NNs.
\item \textbf{Eliminating traversal order.} We design perturbation domains and a recursive refinement procedure that leverage Linear Relaxation based Perturbation Analysis (LiRPA)-based certificates to simultaneously discard multiple irrelevant features. This removes the sequential bottleneck inherent in prior work and yields an abstract minimal explanation.
\item \textbf{Provable quality guarantees.} We provide the first procedure to measure the worst-case gap between abstract minimal explanations and true minimal abductive explanations, combining adversarial search with optional \gls{verixplus} refinement.
\item \textbf{Scalable evaluation.} We benchmark FAME on medium- and large-scale neural networks, 
showing consistent improvements in both explanation size and runtime over \gls{verixplus}. Notably, we produce the first abstract formal abductive explanations for a ResNet architecture on CIFAR-10, demonstrating scalability where exact methods become intractable. 
\end{itemize}


\section{Abductive Explanations \& Verification}
\subsection{Notations}
Scalars are denoted by lower-case letters (e.g., $x$), and the set of real numbers by $\mathbb{R}$. Vectors are denoted by bold lower-case letters (e.g., $\mathbf{x}$), and matrices by upper-case letters (e.g., $W$). The $i$-th component of a vector $\mathbf{x}$ (resp. line of a matrix $W$) is written as $\mathbf{x}_i$ (resp. $W_i$). The matrix $W^{\ge 0}$ (resp. $W^{\le 0}$) represents the same matrix with only nonnegative (resp. nonpositive) weights. Sets are written in calligraphic font (e.g., $\mathcal{S}$). We denote the perturbation domain by $\Omega$ and the property to be verified by $\mathrm{P}$.
%

\subsection{The Verification Context}
\label{sec:verif-context}

We consider a neural network as a function $f : \mathbb{R}^{n} \to \mathbb{R}^k$. The core task of verification is to determine whether the network's output $f(x')$ satisfies a given property $\mathrm{P}$ for every possible input $x'$ within a specified domain $\Omega(x) \subseteq \mathbb{R}^n$. When verification fails, it means there is at least one input $x'$ in the domain $\Omega(x)$ that violates the property $\mathrm{P}$ (a counterexample). 
The verification task can be written as: $
\forall x^{\prime} \in \Omega(x), \text{ does } f(x^{\prime}) \text{ satisfy } \mathrm{P}?
$
This requires defining two components:
\begin{enumerate}
    \item \textbf{The Perturbation Domain ($\Omega$)}: This domain defines the set of perturbations. It is often an $l_p$-norm ball around a nominal input $x$, such as an $l_\infty$ ball for modeling imperceptible noise: $\Omega = \{\mathbf{x'} \in \mathbb{R}^n \mid \|\mathbf{x} - \mathbf{x'}\|_\infty \leq \epsilon\}$.
    
    \item \textbf{The Property ($\mathrm{P}$)}: This is the specification the network must satisfy. For a classification task where the network correctly classifies an input $\mathbf{x}$ into class $c$, the standard robustness property $\mathrm{P}$ asserts that the logit for class $c$ remains the highest for any perturbed input $\mathbf{x'}$:
    \begin{equation}
        \mathrm{P}(\mathbf{x^{\prime}}) \equiv \min_{i \neq c} \left\{ f_c(\mathbf{x^{\prime}}) - f_i(\mathbf{x^{\prime}}) \right\} > 0
    \label{eq:verified-property}
    \end{equation}
\end{enumerate}

\newtext{For instance, given an MNIST image $x$ of a '7' and a perturbation radius $\epsilon$, the property $P$ holds if the network's logit for class '7' provably exceeds all other logits for every perturbed image $x' \in \Omega(x)$.}


A large body of work has investigated formal verification of NNs, with adversarial robustness being the most widely studied property~\citep{urban2021review}. Numerous verification tools are now available off-the-shelf, and for piecewise-linear models $f$ with corresponding input domains and properties, exact verification is possible~\citep{katz2017reluplex, lomuscio2023venus}. In practice, however, exact methods quickly become intractable for realistic networks. \newtext{To address this, we rely on Abstract Interpretation, a theory of sound approximation. Specifically, we utilize Linear Relaxation-based Perturbation Analysis (LiRPA) \citep{zhang2018efficient, singh2019abstract} which efficiently over-approximates the network’s output by enclosing it between linear upper and lower bounds.}
Such abstractions enable sound but conservative verification: if the relaxed property holds, the original one is guaranteed to hold as well. \newtext{We provide a comprehensive background 
in Appendix \ref{appendix-background}.}

\subsection{Abductive Explanations: Pinpointing the "Why"}
\textbf{Understanding Model Robustness with Formal Explanations: }
Neural networks often exhibit sensitivity to minor input perturbations, a phenomenon that certified training can mitigate but not eliminate \citep{de2024using}. Even robustly trained models may only have provably safe regions spanning a few pixels for complex tasks like ImageNet classification \citep{serrurier_achieving_2021}. To build more reliable systems, it is crucial to understand \textit{why} a model's prediction is robust (or not) within a given context. Formal explainability provides a rigorous framework for this analysis.

We focus on \emph{abductive explanations} (AXps, also called \gls{distanceAXp}) \citep{ignatiev2019abduction,huang2023robustness}, which identify a subset of input features that are \textit{sufficient} to guarantee that the property $\mathrm{P}$ holds. Formally, a local formal abductive explanation is defined as a subset of input features that, if collapsed to their nominal values (i.e., the sample \( \mathbf{x} \)), ensure that the local perturbation domain $\Omega$ surrounding the sample contains no counterexamples.


\begin{definition}[\gls{WAXp} \label{def:formal-xai}]
Formally, given a triple $(\mathbf{x}, \Omega, \mathrm{P})$,  
an \emph{explanation} is a subset of feature indices $\mathcal{X} \subseteq \mathcal{F}=\{1, \dots, n\}$ such that
\begin{equation}
\text{\gls{WAXp}:} \;  \forall \mathbf{x^{\prime}} \in \Omega(\mathbf{x}), \quad 
\left( \bigwedge_{i \in \mathcal{X}} (\mathbf{x^{\prime}}_i = \mathbf{x}_i) \right) 
\implies f(\mathbf{x'}) \models \mathrm{P}.
\label{eq:weakformal_xai}
\end{equation}
\end{definition}

While many such explanations may exist (the set of all features \(\mathcal{F}\) is a trivial one), the most useful explanations are the most concise ones~\citep{bassan2023towards}. We distinguish three levels:

\textbf{Minimal Explanation:} An explanation $\mathcal{X}$ is \emph{minimal} if removing any single feature from it would break the guarantee (i.e., $\mathcal{X} \setminus \{j\}$ is no longer an explanation for any $j \in \mathcal{X}$). These are also known as minimal unsatisfiable subsets\citep{ignatiev2016propositional,bassan2023towards}.
    
\textbf{Minimum Explanation:} An explanation $\mathcal{X}$ is \emph{minimum} if it has the smallest possible number of features (cardinality) among all possible minimal explanations. 

The concept of an abductive explanation is illustrated using a classification task (details in Appendix \ref{example-abductive}, Figure \ref{fig:toy_example_main}).  The goal is to find a minimal subset of fixed features ($\mathcal{X}$) that guarantees a sample's classification within its perturbation domain. For the analyzed sample, fixing $\mathbf{x}_2$ alone is insufficient due to the existence of a counterexample (Figure \ref{fig:abductive_example_main}). However, fixing the set $\mathcal{X}=\{\mathbf{x}_2, \mathbf{x}_3\}$ creates a 'safe' subdomain without counterexamples, confirming it is an abductive explanation. This explanation is minimal (neither $\mathbf{x}_2$ nor $\mathbf{x}_3$ work alone) but not minimum in cardinality, as $\mathcal{X}'=\{\mathbf{x}_1\}$ is also a valid minimal explanation. 
In the rest of this paper, we will use the terms abductive explanation or formal explanation and the notation \gls{WAXp} to refer to Definition ~\ref{def:formal-xai}.

\section{Related Work}

Substantial progress has been made in the practical efficiency of computing formal explanations. While finding an abductive explanation ($\AXp$) is tractable for some classifiers \citep{marques2023disproving, darwiche2022computation, huang2022tractable, huang2021efficiently, izza2020explaining, marques2020explaining, marques2021explanations}, it becomes computationally hard for complex models like random forests and neural networks \citep{ignatiev2021sat, izza2021explaining}. To address this inherent complexity, these methods typically encode the problem as a logical formula, leveraging automated reasoners like SAT, SMT, and Mixed Integer Linear Programming (MILP) solvers \citep{audemard2022trading, ignatiev2020towards, ignatiev2022using, ignatiev2021sat, izza2021explaining}
. Early approaches, such as deletion-based \citep{chinneck1991locating} and insertion-based \citep{de88jl} algorithms, are inherently sequential, thus requiring an ordering of the input features traditionally denoted as \textit{traversal ordering}. They require a number of costly verification calls linear with the number of features, which prevents effective parallelization. As an alternative, surrogate models have been used to compute formal explanations for complex models \citep{boumazouza2021asteryx,boumazouza2023symbolic}, but the guarantee does not necessary hold on the original model.

Recent work aims to break the sequential bottleneck, by linking explainability to adversarial robustness and formal verification. DistanceAXp \citep{huang2023robustness,la2021guaranteed} is a key example, aligning with our definition of $\AXp$ and enabling the use of verification tools.

The latest literature focuses on breaking the sequential bottleneck using several strategies that include parallelization. This is achieved either by looking for several counterexamples at once \citep{izza2024distance, bassan2023towards, la2021guaranteed, bassan2023formally} or by identifying a set of irrelevant features simultaneously, as seen in \gls{verix} \citep{wuverix}, \gls{verixplus} \citep{wu2024better}, and prior work \citep{bassan2023towards}. For instance, \gls{verixplus} introduced stronger traversal strategies to alleviate the sequential bottleneck. Their binary search approach splits the remaining feature set and searches for batches of consecutive irrelevant features, yielding the same result as sequential deletion but with fewer solver calls. They also adapted QuickXplain \citep{junker2004quickxplain}, which can produce even smaller explanations at the cost of additional runtime by verifying both halves. 
Concurrently, \citep{bassan2023towards} proposed strategies like the singleton heuristic to reuse verification results and derived provable size bounds, but their approach remains significantly slower than \gls{verixplus} and lacks publicly available code. 
 
The identified limitations are twofold. First, existing methods rely heavily on exact solvers such as Marabou \citep{katz2019marabou, wu2024marabou}, which do not scale to large NNs and are restricted to CPU execution. Recent verification benchmarks \citep{vnncomp2023, ducoffe2024surrogate, zhao2022cleverest} consistently demonstrate that GPU acceleration and distributed verification are indispensable for achieving scalability. Second, these approaches critically depend on traversal order. As shown in \gls{verix}, the chosen order of feature traversal strongly impacts both explanation size and runtime. Yet, determining an effective order requires prior knowledge of feature importance, precisely the information that explanations are meant to uncover, thus introducing a circular dependency.
Nevertheless, \gls{verixplus} currently represents the SOTA for abductive explanations in NNs, achieving the best trade-off between explanation size and computation time.

Our work builds on this foundation by directly addressing the sequential bottleneck of formal explanation without requiring a traversal order, a first in formal XAI. We demonstrate that leveraging incomplete verification methods and GPU hardware is essential for practical scalability. Our approach offers a new solution to the core scalability issues, complementing other methods that aim to reduce explanation cost through different means \citep{bassan2025explain, bassan2025explaining}.

\section{FAME: Formal Abstract Minimal Explanation}
\newtext{In this section, we introduce FAME, a framework that builds \emph{abstract abductive explanations} (Definition~\ref{def:formal-abstract-xai}). FAME  proposes novel strategies to provide sound abstract abductive explanations ($\wAXPA$) such as an Abstract Batch Certificate using Knapsack formulation, and a Recursive Refinement, relying on raw bounds provided by a formal framework (we use LiRPA in this paper). 
}

\begin{definition}[Abstract Abductive Explanation ($\wAXPA$)]\label{def:formal-abstract-xai}
Formally, given a triple $(\mathbf{x}, \Omega, \mathrm{P})$, an \emph{abstract abductive explanation} is a subset of feature indices 
$\mathcal{X}^A \subseteq \mathcal{F} = \{1, \dots, n\}$ such that, under an abstract interpretation $\overline{f}$ of the model $f$, the following holds:
\begin{equation}
\wAXPA: \forall \mathbf{x'} \in \Omega(\mathbf{x}), \quad 
\left( \bigwedge_{i \in \mathcal{X}^A} (\mathbf{x'}_i = \mathbf{x}_i) \right) 
\implies \overline{f}(\mathbf{x'}) \models \mathrm{P}.
\label{eq:formal_xai}
\end{equation}

\noindent Here, $\overline{f}=\text{LiRPA}(f, \Omega)$ denotes the sound over-approximated bounds of the model outputs on the domain $\Omega$, as computed by the LiRPA method.
If Eq.~\eqref{eq:formal_xai} holds, any feature outside $\mathcal{X}^A$ can be considered irrelevant \emph{with respect to the abstract domain}. This ensures that the concrete implication $f(\mathbf{x'}) \models \mathrm{P}$ also holds for all $x' \in \Omega$.  In line with the concept of abductive explanations, we define an \emph{abstract minimal explanation} as an abstract abductive explanation (\wAXPAmin) a set of features $\mathcal{X}^A$ from which no feature can be removed without violating Eq.~\eqref{eq:formal_xai}.
\end{definition}

Due to the over-approximation, as detailed in Section~\ref{sec:verif-context}, any \emph{abstract abductive explanation} is a \emph{weak abductive explanation} for the model $f$.
In the following we present the first steps described in Figure~\ref{fig:fameframework} to build such a $\wAXPA$.

\subsection{The Asymmetry of Parallel Feature Selection}
In the context of formal explanations, \textbf{adding a feature} means identifying it as essential to a model's decision (causes the model to violate the desired property $\mathrm{P}$), so its value must be fixed in the explanation.
Conversely, \textbf{freeing a feature} means identifying it as irrelevant, allowing it to vary without affecting the prediction. A key insight is the asymmetry between these two actions: while adding necessary features can be parallelized naively, freeing features cannot due to complex interactions.

\begin{proposition}
[\textbf{Simultaneous Freeing}]\label{lemma:simultaneous_free} 
it is unsound to free multiple features at once based only on individual verification as two features may be individually irrelevant yet jointly critical.
\end{proposition}

\newtext{Parallelizing feature freeing based on individual verification queries is unsound due to hidden feature dependencies that stem from treating the verifier as a simple binary oracle (SAT/UNSAT; see Appendix~\ref{appendix-background} for formal definitions)
(Proposition~\ref{lemma:simultaneous_free}). To solve this, we introduce the Abstract Batch Certificate $\Phi(\mathcal{A})$ (Definition~\ref{def:batch-certificate}). Unlike naive binary checks, $\Phi(\mathcal{A})$ leverages abstract interpretation to compute a joint upper bound on the worst-case contribution of the entire set $\mathcal{A}$ simultaneously. If $\Phi(\mathcal{A}) \le 0$, it mathematically guarantees that simultaneously freeing $\mathcal{A}$ is sound, explicitly accounting for their combined interactions.} The formal propositions detailing this asymmetry 
is provided in the Appendix \ref{appendix-proof}.


\subsection{Abstract Interpretation for Simultaneous Freeing}
\label{sec:abstract-simult-freeing} 
Standard solvers act as a "binary oracle" and their outcomes (SAT/UNSAT) are insufficient to certify batches of features for freeing without a traversal order. This is because of feature dependencies and the nature of the verification process. 
We address this by leveraging \emph{inexact} verifiers based on abstract interpretation (LiRPA) to extract \emph{proof objects} (linear bounds) that conservatively track the contribution of any feature set. Specifically, we use CROWN \citep{zhang2018efficient} to define an \emph{abstract batch certificate} $\Phi$  in Definition~\ref{def:batch-certificate}. 
If one succeeds in freeing a set of features $\mathcal{A}$ given $\Phi$, we denote such an explanation as a \emph{formal abstract explanation} that satisfies Proposition~\ref{lemma:batch-certificate}. 

\begin{definition}[Abstract Batch Certificate]\label{def:batch-certificate}
Let $\mathcal{A}$ be a set of features and $\Omega$ any perturbation domain. The \emph{abstract batch certificate} is defined as:
\[
\Phi(\mathcal{A}; \Omega) \;=\; \max_{i \neq c}\Big(\overline{b}^i(x) + \sum_{j \in \mathcal{A}} c_{i,j}\Big),
\]

where the baseline bias (worst-case margin of the model's output) at $x$ is
$
\overline{b}^i(x) = \overline{W}^i \cdot x + \overline{w}^i,
$

and the contribution of each feature $j \in \mathcal{A}$ is
$
c_{i,j} \;=\; \max\Big\{\overline{W}^{i,\ge 0}_j\,(\overline{x}_j-x_j),\;
                 \overline{W}^{i,\le 0}_j\,(\underline{x}_j-x_j)\Big\},
$
with 
\(\overline{x}_j = \max\{x'_j : x' \in \Omega(x)\}\) and
\(\underline{x}_j = \min\{x'_j : x' \in \Omega(x)\}\).  
The weights $\overline{W}^i$ and biases $\overline{w}^i$ are obtained from LiRPA bounds, which guarantee for each target class $i\neq c$, with $c$ being the groundtruth class:
\[
\forall x' \in \Omega(x), \quad
f_i(x') - f_c(x') \le \overline{f}_{i,c}(x') \;=\; \overline{W}^i \cdot x' + \overline{w}^i,
\]

\end{definition}

\begin{proposition}[\textbf{Batch-Certifiable Freeing}]\label{lemma:batch-certificate}
If $\Phi(\mathcal{A}; \Omega) \le 0$, then $\mathcal{F} \setminus \mathcal{A}$ is a weak abductive explanation (\gls{WAXp}).
\end{proposition}

\begin{lemma}
If $\Phi(\mathcal{A}) \le 0$, freeing all features in $\mathcal{A}$ is sound; that is, the property $\mathrm{P}$ holds for every $x' \in \Omega(x)$ with \(\{x'_k = x_k\}_{k \in \mathcal{F} \setminus \mathcal{A}}\). 
\end{lemma}

The proof of Proposition~\ref{lemma:batch-certificate} is given in Appendix \ref{appendix-proof}. 
The trivial case $\mathcal{A} = \emptyset$ always satisfies the certificate, but our goal is to efficiently certify large feature sets. The abstract batch certificate also highlights two extreme scenarii. In the first, if $\Phi(\mathcal{F}) \le 0$, all features are irrelevant, meaning the property $\mathrm{P}$ holds across $\Omega$ without fixing any inputs. In the second, if $\overline{b}^i(x) \ge 0$ for some $i \neq c$, then $\Phi(\emptyset) > 0$ and no feature can be safely freed; this situation arises when the abstract relaxation is too loose, producing vacuous bounds. Avoiding this degenerate case requires careful selection of the \emph{perturbation domain}, a consideration we highlight for the first time in the context of abductive explanations. The choice of abstract domain is discussed in Section~\ref{sec:perturbation_domain}.

\subsection{Minimizing the Size of an Abstract Explanation via a Knapsack Formulation}
Between the trivial and degenerate cases lies the nontrivial setting: finding a maximal \emph{set of irrelevant features} $\mathcal{A}$ to free given the abstract batch certificate $\Phi$. Let $\mathcal{F}$ denote the index set of features. Maximizing $|\mathcal{A}|$ can be naturally formulated as a $0/1$ Multidimensional Knapsack Problem (MKP). For each feature $j \in \mathcal{F}$, we introduce a binary decision variable $y_j$ indicating whether the feature is selected. The optimization problem then reads:
\begin{equation}\label{eq:RP}
\max_{y} \sum_{j \in \mathcal{F}} y_j 
\;\text{s.t.}\; 
\sum_{j \in \mathcal{F}} c_{ij} y_j \le -\overline{b}^i(x), \quad i \in I, \, i \neq c
\end{equation}
where $c_{i,j}$ represents the contribution of feature $j$ to constraint $i$, and $-\overline{b}^i(x)$ is the corresponding knapsack capacity. The complexity of this MKP depends on the number of output classes. For binary classification ($k=2$), the problem is linear\footnotemark \footnotetext{it can be solved optimally in $\mathcal{O}(n)$ time by sorting features by ascending contribution $c_{1,j}$ and greedily adding them until the capacity is exhausted.}. In the standard multiclass setting ($k>2$), however, the MKP is NP-hard. While moderately sized instances can be solved exactly using a MILP solver, this approach does not scale to large feature spaces. To ensure scalability, we propose a simple and efficient greedy heuristic, formalized in Algorithm~\ref{alg:abstract-batch}. Rather than solving the full MKP, the heuristic iteratively selects the feature $j^\star$ that is \emph{least likely} to violate any of the $k-1$ constraints, by minimizing the maximum normalized cost across all classes. An example is provided in Appendix \ref{exp:greedy-knapsack}. This procedure is highly parallelizable, since all costs can be computed simultaneously. While suboptimal by design, it produces a set $\mathcal{A}$ such that $\Phi(\mathcal{A}; \Omega) \leq 0$. 
\newtext{A key advantage of this greedy batch approach is its computational efficiency. The cost is dominated by the computation of feature contributions $c_{i,j}$. This requires a single backward pass through the abstract network, which has a complexity of $O(L \cdot N)$ (where $L$ is depth and $N$ is neurons) and is highly parallelizable on GPUs. In contrast, exact solvers require solving an NP-hard problem for each feature or batch.} 
In Section~\ref{sec:experiments}, we compare the performance of this greedy approach against the optimal MILP solution, demonstrating that it achieves competitive results with dramatically improved scalability.

\begin{algorithm}[H]
\caption{Greedy Abstract Batch Freeing (One Step)}\label{alg:abstract-batch}
\begin{algorithmic}[1]
\State \textbf{Input:} model $f$, perturbation domain $\Omega^m$, candidate set $F$
\State \textbf{Initialize:} $\mathcal{A}\gets\emptyset$, linear bounds $\{\overline{W}^i,\overline{w}^i\}=\text{LiRPA}(f, \Omega^m(x))$
\State \textbf{Do:} compute $c_{i,j}$ in parallel
\While{$\Phi(\mathcal{A})\le 0$ and $|\mathcal{F}| >0$}
    \State pick $j^\star = \arg\min_{j\in F\setminus \mathcal{A}} \max_{i\ne c}\, c_{i,j}/(-\overline{b}_i)$ \Comment{Parallel reduction}
    \If{$\Phi(\mathcal{A}\cup\{j^\star\}) \le 0$ and $|\mathcal{A}|\le m$} \State $\mathcal{A}\gets \mathcal{A}\cup\{j^\star\}$
    \EndIf
    \State $F \gets F \setminus \{j^\star\}$ \Comment{Remove candidate}
\EndWhile
\State \textbf{Return:} $\mathcal{A}$
\end{algorithmic}
\end{algorithm}

\section{Refining the Perturbation Domain for Abductive Explanation}\label{sec:perturbation_domain}

Previous approaches for batch freeing reduce the perturbation domain using a traversal order $\pi$, defining
$\Omega_{\pi,i}(\mathbf{x}) = \{\mathbf{x'} \in \mathbb{R}^n : \|\mathbf{x} - \mathbf{x'}\|_\infty \le \epsilon, \; \mathbf{x'}_{\pi_{i:}} = \mathbf{x}_{\pi_{i:}}\}.
$
These methods only consider freeing dimensions up to a certain order. However, as discussed previously, determining an effective order requires prior knowledge of feature importance, the very information that explanations aim to uncover, introducing a circular dependency. This reliance stems from the combinatorial explosion: the number of possible subsets of input features grows exponentially, making naive enumeration of abstract domains intractable.

To address this, we introduce a new perturbation domain, denoted the \emph{cardinality-constrained perturbation domain}. For instance, one can restrict to $\ell_0$-bounded perturbations:
\[
\Omega^m(\mathbf{x}) = \{\mathbf{x'} \in \mathbb{R}^n : \|\mathbf{x} - \mathbf{x'}\|_\infty \le \epsilon, \; \|\mathbf{x} - \mathbf{x'}\|_0 \le m\},
\] 
which ensures that at most $m$ features may vary simultaneously. 
This concept is closely related to the $\ell_0$ norm and has been studied in verification \citep{xu2020automatic}, but, to the best of our knowledge, it is applied here for the first time in the context of abductive explanations. The greedy procedure in Algorithm~\ref{alg:abstract-batch} can then certify a batch of irrelevant features $\mathcal{A}$ under this domain. Once a set $\mathcal{A}$ is freed, the feasible perturbation domain becomes strictly smaller, enabling tighter bounds and the identification of additional irrelevant features. We formalize this as the \emph{refined abstract domain} that ensures that at most $m$ features can vary in addition to the set of previously seclected ones $\mathcal{A}$:
\begin{equation}
\label{eq:refined-domain}
\Omega^m(\mathbf{x}; \mathcal{A}) = \{\mathbf{x'} \in \mathbb{R}^n : \|\mathbf{x} - x'\|_\infty \le \epsilon, \; \|\mathbf{x}_{\mathcal{F}\setminus \mathcal{A}} - \mathbf{x'}_{\mathcal{F}\setminus \mathcal{A}}\|_0 \le m\}.
\end{equation}
By construction, $\Omega^m(\mathbf{x}; \mathcal{A}) \subseteq \Omega^{m+|\mathcal{A}|}(\mathbf{x})$, 
so any free set derived from $\Omega^m(x; \mathcal{A})$ remains sound for the original budget $m+|\mathcal{A}|$. 
Recomputing linear bounds on this tighter domain often yields strictly smaller abstract explanation. This refinement naturally suggests a recursive strategy: after one round of greedy batch freeing, we restrict the domain to $\Omega^m(\mathbf{x}; \mathcal{A})$, recompute LiRPA bounds, and reapply Algorithm~\ref{alg:abstract-batch} for $m = 1 \dots |\mathcal{F}\setminus \mathcal{A}|$. 
\newtext{Unlike the static traversal of prior work (e.g., VERIX+), FAME employs a dynamic, cost-based selection by re-evaluating abstract costs $c_{i,j}$ at each recursive step. This process functions as an adaptive abstraction mechanism: iteratively enforcing cardinality constraints tightens the domain, reducing LiRPA's over-approximation error and enabling the recovery of additional freeable features initially masked by loose bounds.} 
As detailed in Algorithm~\ref{alg:recursive}, this process can be iterated, progressively shrinking the domain and expanding $\mathcal{A}$. In practice, recursion terminates once no new features can be freed. Finally, any remaining candidate features can be tested individually using the binary search approach proposed by VeriX+ but replacing Marabou by CROWN (see Algorithm~\ref{alg:singleton_free}). This final step ensures that we obtain a formal abstract minimal explanation, as defined in Definition~\ref{def:formal-abstract-xai}

\begin{algorithm}[H]
\caption{Recursive Abstract Batch Freeing}\label{alg:recursive}
\begin{algorithmic}[1]
\State \textbf{Input:} model $f$, input $x$, candidate set $\mathcal{F}$
\State \textbf{Initialize:} $\mathcal{A} \gets \emptyset$ \Comment{certified free set}
\Repeat
    \State $\mathcal{A}_{best} \gets \emptyset$
    \For{$m = 1 \dots |\mathcal{F} \setminus \mathcal{A}|$}
        \State $\mathcal{A}_m \gets$ \Call{GreedyAbstractBatchFreeing}{$f, \Omega^m(x; \mathcal{A}), \mathcal{F}\setminus\mathcal{A}$}
        \If{$|\mathcal{A}_m| > |\mathcal{A}_{best}|$}
            \State $\mathcal{A}_{best} \gets \mathcal{A}_m$
        \EndIf
    \EndFor
    \State $\mathcal{A} \gets \mathcal{A} \cup \mathcal{A}_{best}$
\Until{$\mathcal{A}_{best} = \emptyset$}
\State $\mathcal{A}$ = \Call{Iterative Singleton Free}{f, x, $\mathcal{F}$, $\mathcal{A}$}\Comment{refine by testing remaining features}
\State \textbf{Return:} $\mathcal{A}$
\end{algorithmic}
\end{algorithm}

\section{Distance from Abstract Explanation to Minimality}
\label{sec:distance-to-minimality}
Algorithm~\ref{alg:recursive} returns a \emph{minimal abstract explanation}: with respect to the chosen LiRPA relaxation, the certified free set $\mathcal{A}$ cannot be further enlarged. This guarantee is strictly weaker than minimality in the exact sense. The remaining features may still include irrelevant coordinates that abstract interpretation fails to certify, due to the coarseness of the relaxation. In other words, minimality is relative to the verifier: stronger but more expensive verifiers (e.g., Verix+ with Marabou) are still required to converge to a \emph{true minimal explanation}.

\newtext{We achieve this via a two-phase pipeline (Figure \ref{fig:fameframework}). \textbf{Phase 1 (Abstract Pruning)} generates a sound abstract explanation \wAXPAmin. \textbf{Phase 2 (Exact Refinement)} minimizes this candidate using VERIX+, ensuring the final output is guaranteed minimal.} 
The gap arises from the tradeoff between verifier accuracy and domain size. Abstract methods become more conservative as the perturbation domain grows, while exact methods remain sound but scale poorly. This motivates hybrid strategies that combine fast but incomplete relaxations with targeted calls to exact solvers. As an additional acceleration step, adversarial attacks can be used. By Lemma~\ref{lemma:simultaneous_add}, if attacks identify features that must belong to the explanation, they can be added simultaneously (\textit{see Algorithm~\ref{alg:batch_add}}). Unlike abstract interpretation, the effectiveness of adversarial search typically increases with the domain size: larger regions make it easier to find counterexamples.

\paragraph{Towards minimal explanations.}
\newtext{In formal XAI, fidelity is a hard constraint guaranteed by the verifier. Therefore, the explanation cardinality (minimality) becomes the only metric to compare formal abductive explanations. A smaller explanation is strictly better, provided it remains sufficient.} Our strategy is to use the \textit{minimal abstract explanation} (\wAXPAmin) as a starting point, and then search for the closest minimal explanation. Concretely, we aim to identify the largest candidate set of potentially irrelevant features that, if freed together, would allow all remaining features to be safely added to the explanation at once.  
A good traversal order of the candidate space is crucial here, as it determines how efficiently such irrelevant features can be pinpointed. Formally, if $\mathcal{X}^A$ denotes the minimal abstract explanation and $\mathcal{X}^{A^\star}$ the closest minimal explanation, we define the \emph{absolute distance to minimality} as the number of irrelevant features not captured by the abstract method:
$
d(\mathcal{X}^A, \mathcal{X}^{A^\star}) \;=\; \big|\mathcal{X}^A \setminus \mathcal{X}^{A^\star}\big|.
$

\section{Experiments}\label{sec:experiments}
To evaluate the benefits and reliability of our proposed explainability method, FAME, we performed a series of experiments comparing its performance against the SoTA \gls{verixplus} implementation. We assessed the quality of the explanations generated by FAME by comparing them to those of \gls{verixplus} across four distinct models, including both fully connected and convolutional neural networks (CNNs). We considered two primary performance metrics: the runtime required to compute a single explanation and the size (cardinality) of the resulting explanation. 

Our experiments, as in \gls{verixplus} \citep{wu2024better}, were conducted on two widely-used image classification datasets: MNIST \citep{yann2010mnist} and GTSRB \citep{stallkamp2012man}. Each score was averaged over non-robust samples from the 100 samples of each dataset. For the comparison results, the explanations were generated using the FAME framework only, and with a final run of  \gls{verixplus} to ensure minimality (See Figure \ref{fig:fameframework}).


\begin{table}[h!]
\centering
\resizebox{\linewidth}{!}{

\begin{tabular}{l|ll|ll|ll|ll|ll|lll}
\hline
\multicolumn{1}{c|}{} & \multicolumn{2}{c|}{\gls{verixplus} (alone) } & \multicolumn{4}{c|}{\makecell{FAME: Single-round}} & \multicolumn{4}{c|}{\makecell{FAME: Iterative refinement}} & \multicolumn{3}{c}{FAME-accelerated \gls{verixplus}} \\
\rule{0pt}{1.05em}
\makecell{Traversal order} & \multicolumn{2}{c|}{\makecell{bounds}} &  \multicolumn{4}{c|}{/} & \multicolumn{4}{c|}{/} & \multicolumn{3}{c}{\makecell{/ + bounds}} \\
\rule{0pt}{1.em}
\makecell{Search procedure} & \multicolumn{2}{l|}{\makecell{binary}} & \multicolumn{2}{l|}{MILP} & \multicolumn{2}{l|}{Greedy} & \multicolumn{2}{l|}{MILP} & \multicolumn{2}{l|}{Greedy} & \multicolumn{3}{l}{\makecell{Greedy + binary}} \\
\rule{0pt}{1.em}
\makecell{Metrics \textdownarrow} & $|$\textit{AXp}$|$ & time & $|$\textit{$\wAXPA$}$|$  & time   &  $|$\textit{$\wAXPA$}$|$  & time  &  $|$\textit{$\wAXPA$}$|$  & time  & $|$\textit{$\wAXPA$}$|$  & time  & $\|$candidate-set$\|$ & $|\textit{AXp}|$  & time \\
\hline
\hline
\rule{0pt}{1.05em}
\textbf{MNIST-FC} & 280.16&13.87 &441.05&4.4 &448.37 &0.35 &229.73 & 14.30& 225.14& 8.78&44.21 &\textbf{224.41} &  \textbf{13.72}\\
\rule{0pt}{1.05em}
\textbf{MNIST-CNN} &159.78 & 56.72& 181.24& 5.59&190.29 &0.51 & 124.9& 12.35& 122.09&5.6& 104.09& \textbf{113.53}&\textbf{33.75} \\
\rule{0pt}{1.05em}
\textbf{GTSRB-FC} & \textbf{313.42}& 56.18& 236.85& 9.68&243.18 &0.97 &331.84 & 12.28& 332.74& 5.26& 11.93& {332.66}&\textbf{9.26} \\
\rule{0pt}{1.05em}
\textbf{GTSRB-CNN} & 338.28& 185.03&372.66 &12.45 & 379.34&1.35 &322.42& 17.63& 322.42& 7.42 & 219.57&\textbf{322.42} & \textbf{138.12}\\
\hline
\end{tabular}

}
 \caption{Average explanation size and generation time (in seconds) are compared for FAME (single-round and iterative MILP/Greedy) 
with FAME-accelerated \gls{verixplus} to achieve minimality. 
}\label{tab:summary-results}
\end{table}



\textbf{Experimental Setup}
All experiments were carried out on a machine equipped with an Apple M2 Pro processor and $16$~GB of memory. The analysis is conducted on fully connected (-FC) and convolutional (-CNN) models from the MNIST and GTSRB datasets, with $\epsilon$ set to 0.05 and 0.01 respectively. The verified perturbation analysis was performed using the DECOMON library\footnote{https://github.com/airbus/decomon}, applying the CROWN method with an $l_{\infty}$-norm. The NN verifier Marabou \citep{katz2019marabou} is used within \gls{verixplus}. \newtext{We included a sensitivity analysis covering: (1) Solver Choice, confirming the Greedy heuristic's near-optimality vs. MILP (Table \ref{tab:summary-results}); (2) Cardinality Constraints, showing that card=True yields significantly smaller explanations (Figure \ref{fig:fame_verix_comparison}); and (3) Perturbation Magnitude ($\epsilon$), which we fixed to baseline used by \gls{verixplus} for direct comparison.} 
\newtext{We include additional experimental results on the ResNet-2B architecture (CIFAR-10) from the VNN-COMP benchmark \citep{wang2021betacrown} to demonstrate scalability on deeper models.
} The complete set of hyperparameters and the detailed architectures of the models used 
are provided in Appendix \ref{appendix-xps} for full reproducibility.

\begin{figure*}[t]
    \centering
    \begin{minipage}[c]{0.48\textwidth}
        \centering
        \includegraphics[width=0.93\textwidth]{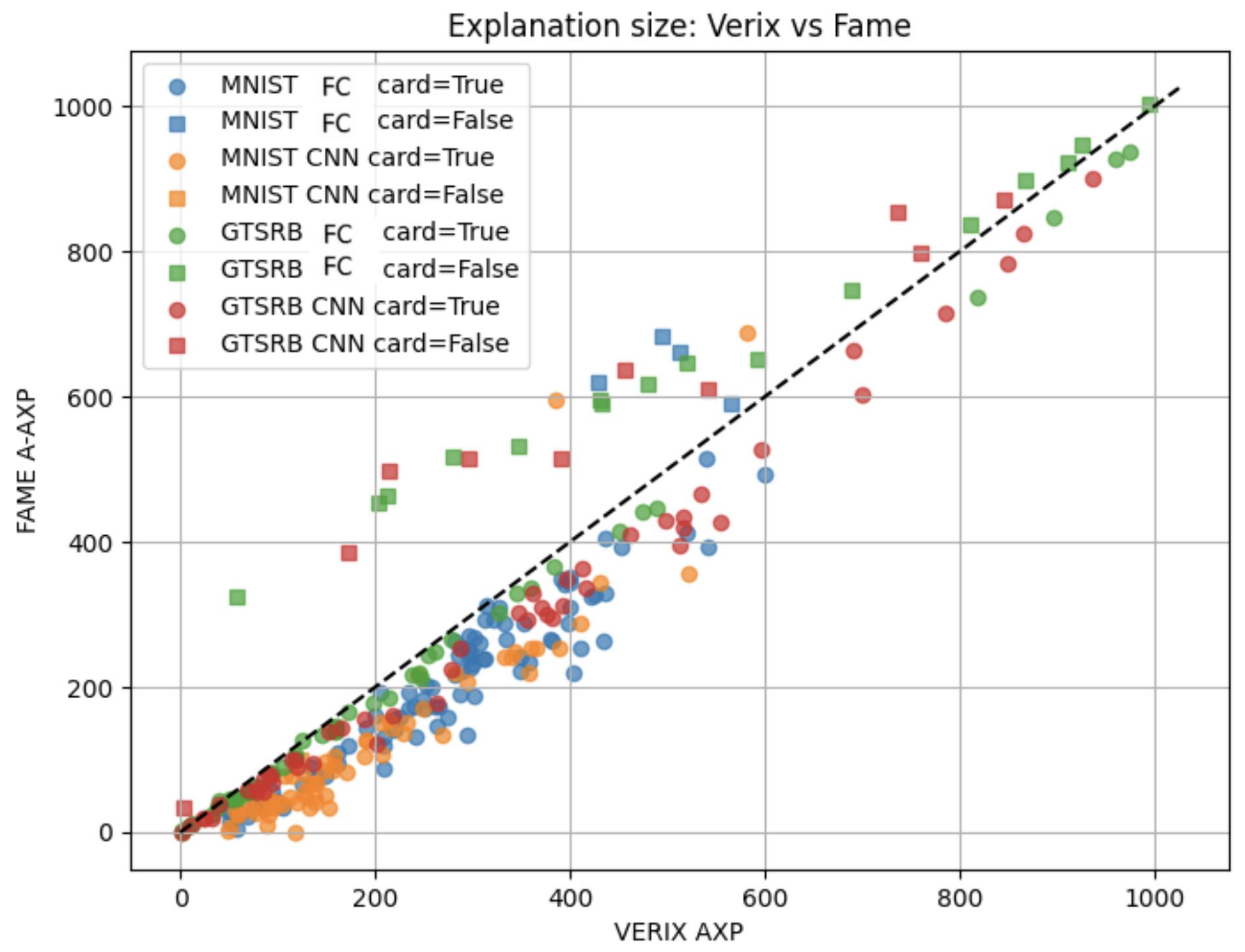}
        \label{fig:size_comparison}
    \end{minipage}
    \begin{minipage}[c]{0.48\textwidth}
        \centering
        \includegraphics[width=0.9\textwidth]{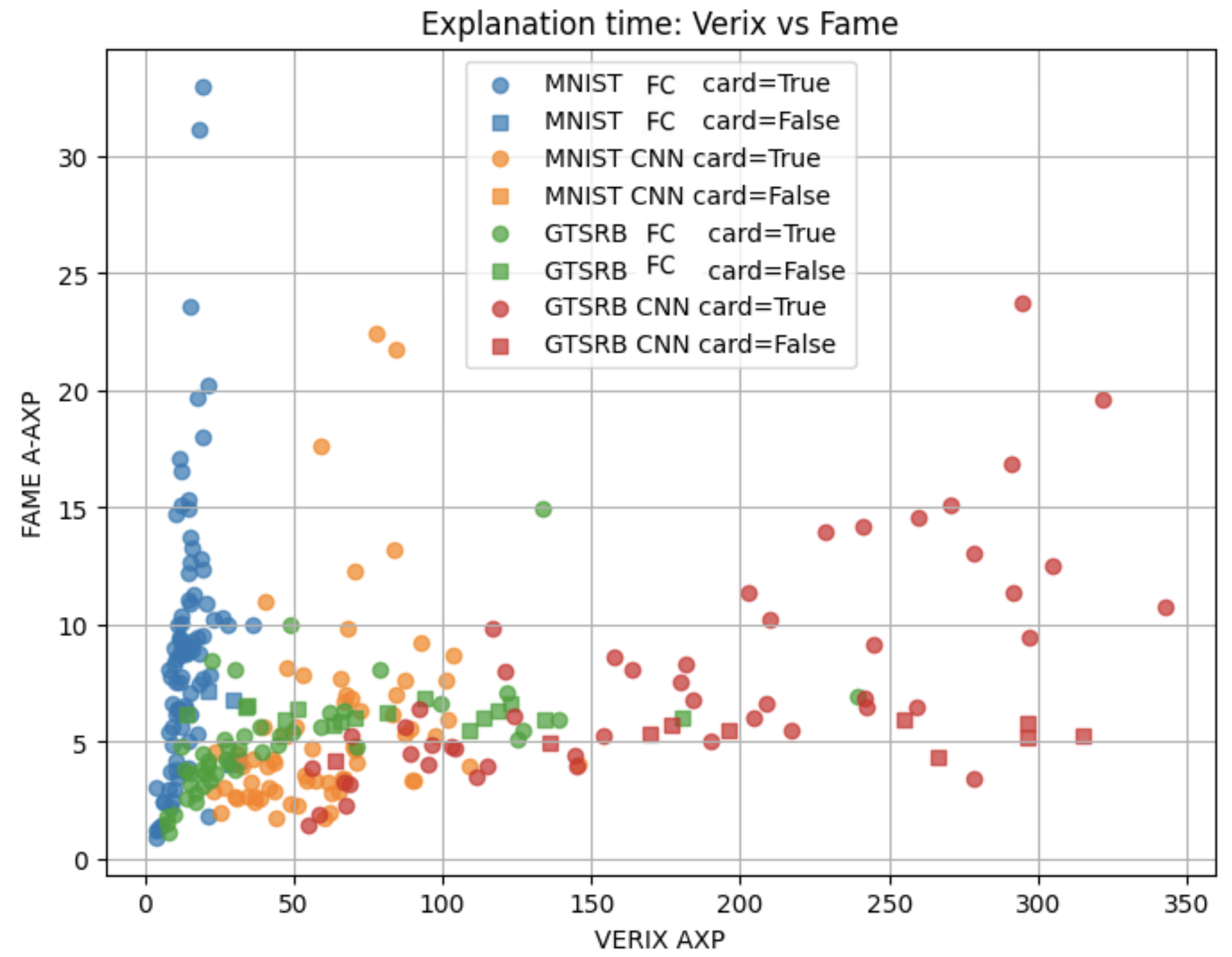}
        \label{fig:time_comparison}
    \end{minipage}
    \caption{ \textbf{FAME's iterative refinement approach against the \gls{verixplus} baseline}. 
    The left 
    plot compares the size of the final explanations. The right
    plot compares the runtime (in $seconds$).  The data points for each model are distinguished by color, and the use of circles (card=True) and squares (card=False) indicates whether a cardinality constraint ($||x-x'||_{0}\le m$) was applied.
        }
    \label{fig:fame_verix_comparison}
\end{figure*}

\subsection{Greedy vs. MILP for Abstract Batch Freeing}

\textbf{Performance in a Single Round} 
This experiment, in the 'FAME: Single Round' column of Table \ref{tab:summary-results}, compares the runtime and size of the largest free set obtained in a single round using the greedy method versus an exact MILP solver for the abstract batch freeing (Algorithm~\ref{alg:abstract-batch}).

Across all models, the greedy heuristic consistently provided a significant \textbf{speedup (ranging from $9\times$ to $12\times$)} while achieving an abstract explanation size very close (fewer than 9 features in average) to that of the optimal MILP solver. This demonstrates that, for single-round batch freeing, the greedy method offers a more practical and scalable solution.


\textbf{Performance with Iterative Refinement} 
This experiment compares the two methods in an iterative setting of the abstract batch freeing, where the perturbation domain is progressively refined (Section~\ref{sec:perturbation_domain}). For the iterative refinement process, the greedy approach maintained a substantial runtime advantage over the MILP solver, with a speedup up to $2.4\times$ on the GTSRB-CNN model, while producing abstract explanations that were consistently close in size to the optimal solution. The distinction between the circle and square markers is significant in Figure \ref{fig:fame_verix_comparison}. The square markers (card=False) tend to lie closer to or even above the diagonal line. This suggests that the cardinality-constrained domain, when successful, is highly effective at finding more compact explanations. 

\newtext{
\textbf{Impact of Iterative Refinement: } Comparing 'FAME: Single-round' vs. 'FAME: Iterative refinement' in Table \ref{tab:summary-results} isolates the impact of Algorithm \ref{alg:recursive}. For MNIST-CNN, iterative refinement reduces explanation size by ~36\% (190.29 to 122.09). This highlights the trade-off: a modest increase in runtime yields significantly more compact explanations.
}
\subsection{Comparison with State-of-the-Art (\gls{verixplus}) }
We compare in this section the results of \gls{verixplus} (alone) vs. FAME-accelerated \gls{verixplus}.

\textbf{Explanation Size and Runtime:} FAME consistently produces smaller explanations than \gls{verixplus} while being significantly faster, mainly due to  FAME's iterative refinement approach, as visually confirmed by the plots in Figure \ref{fig:fame_verix_comparison} that show a majority of data points falling below the diagonal line for both size and time comparisons. 
The runtime gains are particularly substantial for the GTSRB models (green and red markers), where FAME's runtime is often only a small fraction of \gls{verixplus}'s as shown in Table \ref{tab:summary-results}. 
In some cases, FAME delivers a non-minimal set that is smaller than \gls{verixplus} 's minimal set, with up to a $25\times$ speedup  (322.42 features in 7.4s compared to 338.28 in 185.03s for the GTSRB-CNN model) while producing \gls{WAXp}$^A$ that were consistently close in size to the optimal solution.

\textbf{The Role of Abstract Freeing:} The effectiveness of FAME's approach is further supported by the "distance to minimality" metric. The average distance to minimality was 44.21 for MNIST-FC and 104.09 for MNIST-CNN. 
An important observation from our experiments is that when the abstract domains in FAME are effective, they yield abstract abductive explanations $\wAXPA$ that are smaller than the abductive explanations ($\AXp$) from \gls{verixplus}. This is not immediately obvious from the summary table, as the final explanations may differ. Conversely, when FAME's abstract domains fail to find a valid free set, our method defaults to a binary search approach similar to \gls{verixplus}. However, since we do not use the Marabou solver in this phase, the resulting $\wAXPA$ is larger than the $\AXp$ provided by Marabou. This highlights the trade-off and the hybrid nature of our approach.

\newtext{Finally, to demonstrate the generality of our framework beyond standard benchmarks, in Appendix \ref{appendix:cifar_resnet} we provide additional experiments on the ResNet-2B architecture \cite{wang2021betacrown} trained on CIFAR-10. These results represent, to the best of our knowledge, the first formal explanations generated for such a complex architecture, highlighting FAME as an enabling technology for scalability.}

\section{Conclusion and Discussion}\label{sec:conclusion}
In this work, we introduced \textbf{FAME} (Formal Abstract Minimal Explanations), a novel framework for computing abductive explanations that effectively scales to large neural networks. By leveraging a hybrid strategy grounded in abstract interpretation and dedicated perturbation domains, we successfully addressed the long-standing sequential bottleneck of traditional formal explanation methods.

Our main contribution is a new approach that eliminates the need for traversal order by progressively shrinking dedicated perturbation domains and using LiRPA-based bounds to efficiently discard irrelevant features. The core of our method relies on a greedy heuristic for batch freeing that, as our analysis shows, is significantly faster than an exact MILP solver while yielding comparable explanation sizes. 
 
Our experimental results demonstrate that the full hybrid FAME pipeline outperforms the current state-of-the-art \gls{verixplus} baseline, providing a superior trade-off between computation time and explanation quality. We consistently observed significant reductions in runtime while producing explanations that are close to true minimality. This success highlights the feasibility of computing formal explanations for larger models and validates the effectiveness of our hybrid strategy.

Beyond its performance benefits, the FAME framework is highly generalizable. Although our evaluation focused on classification tasks, the framework can be extended to other machine learning applications, such as regression. \newtext{While we focused on robustness in continuous domains, FAME's high-level algorithms (batch certificate, greedy selection) support discrete features (see Appendix \ref{appendix-proof-discrete}). LiRPA natively handles discrete variables (e.g., one-hot encodings) via contiguous interval bounds. Furthermore, the framework can support other properties 
like 
}  local stability. 
Additionally, FAME can be configured to 
use exact solvers for the final refinement step, ensuring its adaptability and robustness for various use cases.

\newtext{Finally, we demonstrated FAME's scalability on the ResNet-2B (CIFAR-10) architecture. Although the abstraction gap naturally widens with depth, FAME's ability to rapidly prune irrelevant features establishes it as a critical enabling step for applying formal XAI to complex models where exact-only methods are currently intractable. By designing a framework that natively leverages certificates from modern, GPU-enabled verifiers, this work effectively bridges the gap between formal guarantees and practical scalability.}






\section*{Acknowledgements}
Our work has benefited from the AI Cluster ANITI and the research program DEEL.\footnote{\url{https://www.deel.ai/}} ANITI is funded by the France 2030 program under the Grant agreement n°ANR-23-IACL-0002. DEEL is an integrative program of the AI Cluster ANITI, designed and operated jointly with IRT Saint Exupéry, with the financial support from its industrial and academic partners and the France 2030 program under the Grant agreement n°ANR-10-AIRT-01. Within the DEEL program,  we are especially grateful to Franck MAMALET for their constant encouragement, valuable discussions, and insightful feedback throughout the development of this work. 
The work of Elsaleh, Bassan, and Katz was partially funded by the European Union
(ERC, VeriDeL, 101112713). Views and opinions expressed
are however those of the author(s) only and do not necessarily reflect those of the European Union or the European
Research Council Executive Agency. Neither the European
Union nor the granting authority can be held responsible for
them. The work of Elsaleh, Bassan, and Katz was additionally supported by a grant
from the Israeli Science Foundation (grant number 558/24). 
Elsaleh is also supported by the Ariane de Rothschild Women Doctoral Program.

\bibliographystyle{iclr2026_conference}
\bibliography{refs.bib}
\newpage
\appendix
\begin{center}
    {\LARGE Appendix}
\end{center}
\vspace{1em}

The appendix collects proofs, model specifications, and supplementary experimental results that support the main paper.

\noindent
\textbf{Appendix A} contains additional background on formal verification terminology, Abstract Interpretation, and LiRPA. \\
\textbf{Appendix B} contains the complete proofs of all propositions. \\
\textbf{Appendix C} provides the pseudocode for the FAME algorithms and the associated baselines. \\
\textbf{Appendix D} provides illustrative examples of abductive explanations and the greedy knapsack formulation. \\
\textbf{Appendix E} provides specifications of the datasets and architectures used, along with supplementary experimental results. \\
\textbf{Appendix F} details the scalability analysis on complex architectures (ResNet-2B on CIFAR-10).
\textbf{Appendix G} provides the LLM usage disclosure.

\vspace{2em}
\section{Background on Formal Verification}\label{appendix-background}

\subsection{Abstract Interpretation}
\newtext{
Abstract Interpretation is a theory of sound approximation of the semantics of computer programs. In the context of neural networks, it allows us to compute over-approximations of the network's output range without executing the network on every single point in the input domain (which is infinite).

While exact verification methods (like MILP solvers) provide precise results, they are generally NP-hard and do not scale to large networks. Abstract interpretation trades precision for scalability (typically polynomial time) by operating on abstract domains (e.g., intervals, zonotopes, or polyhedra) rather than concrete values.
}
\subsection{LiRPA (Linear Relaxation-based Perturbation Analysis)}
\newtext{
LiRPA (Linear Relaxation-based Perturbation Analysis) is a specific, efficient instance of abstract interpretation designed for neural networks. Instead of propagating simple intervals (which become too loose/imprecise in deep networks), LiRPA propagates linear constraints. For every neuron $x_j$, it computes two linear bounds relative to the input $x$:
$$\underline{w}_j^T x + \underline{b}_j \leq f_j(x) \leq \overline{w}_j^T x + \overline{b}_j$$

These linear bounds allow us to rigorously bound the "worst-case" behavior of the network much more tightly than simple intervals. If the lower bound of the correct class minus the upper bound of the target class is positive, we have a mathematically sound certificate of robustness.

\paragraph{Illustrative Example:}

Consider a nominal input image $\bar{x}$ from the MNIST dataset depicting the digit '7'. In a standard local robustness verification task, we define the input domain $\Omega(\bar{x})$ as an $l_\infty$-norm ball with a radius of $\epsilon = 0.05$. This implies that each pixel $x_i$ in the image is permitted to vary independently within the interval $[\bar{x}_i - 0.05, \bar{x}_i + 0.05]$. 

The verification objective is to prove that the property $P$ holds: specifically, that for every possible perturbed image $x \in \Omega(\bar{x})$, the network's output logit for the ground-truth class ('7') remains strictly greater than the logit for any target class $k$ (e.g., '1').  In the context of LiRPA, this is verified by computing a sound lower bound for the correct class ($\underline{f}_7$) and a sound upper bound for the competing class ($\overline{f}_1$). If the verified margin $\underline{f}_7 - \overline{f}_1 > 0$, the network is guaranteed to be robust against all perturbations in $\Omega(\bar{x})$.
}

\subsection{Verification Terminology}
We formulate the check for explanation sufficiency as a constraint satisfaction problem. A query is SAT if a valid perturbation (counter-example) exists, and UNSAT if no such perturbation exists (meaning the explanation is valid).

\begin{itemize}
    \item Soundness (No False Positives): A verifier is sound if it guarantees that any certified property is truly holds. In Abstract Interpretation, soundness is achieved because the computed abstract bounds strictly enclose the true concrete values. If these conservative bounds satisfy the property (UNSAT), the actual network must also satisfy it.
    
    \item Completeness (No False Negatives): A verifier is complete if it is capable of certifying \textit{any} valid explanation. Exact solvers (like MILP) are complete. In contrast, Abstract Interpretation is \textbf{incomplete}: due to over-approximation, the bounds may be too loose to prove a true property, leading to a "don't know" state where the explanation is valid, but the verifier cannot prove it.
\end{itemize}

\section{Proof}\label{appendix-proof}

\subsubsection*{THE ASYMMETRY OF PARALLEL FEATURE SELECTION}

\begin{proposition}[\textbf{Simultaneous Addition}]  \label{lemma:simultaneous_add}
    Any number of essential features can be added to the explanation \textbf{simultaneously}. This property allows us to leverage solvers capable of assessing \textit{multiple verification queries in parallel}, leading to a substantial reduction in runtime.
\end{proposition}


\begin{figure}[h!]
    \centering
    \begin{minipage}{0.48\textwidth}
        \centering
        \includegraphics[width=0.9\textwidth]{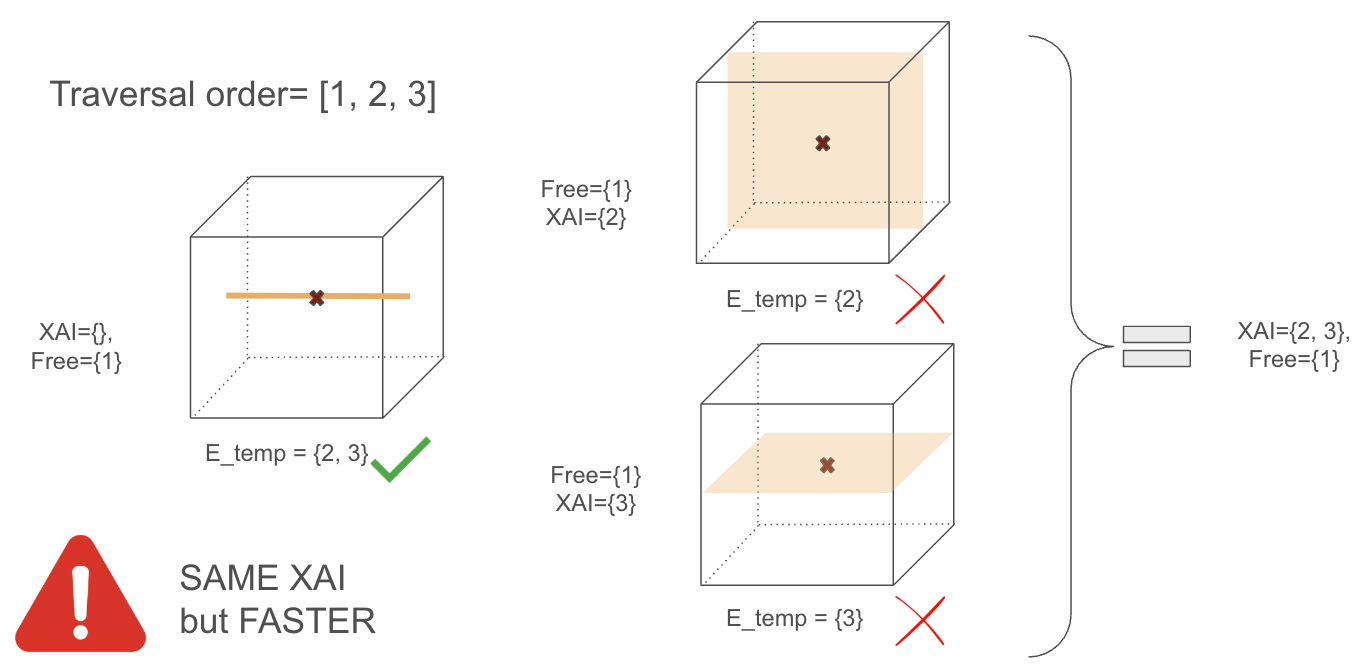}
        \text{(a) Adding several features at once is sound.}
    \end{minipage}
    \hfill
    \begin{minipage}{0.48\textwidth}
        \centering
        \includegraphics[width=0.8\textwidth]{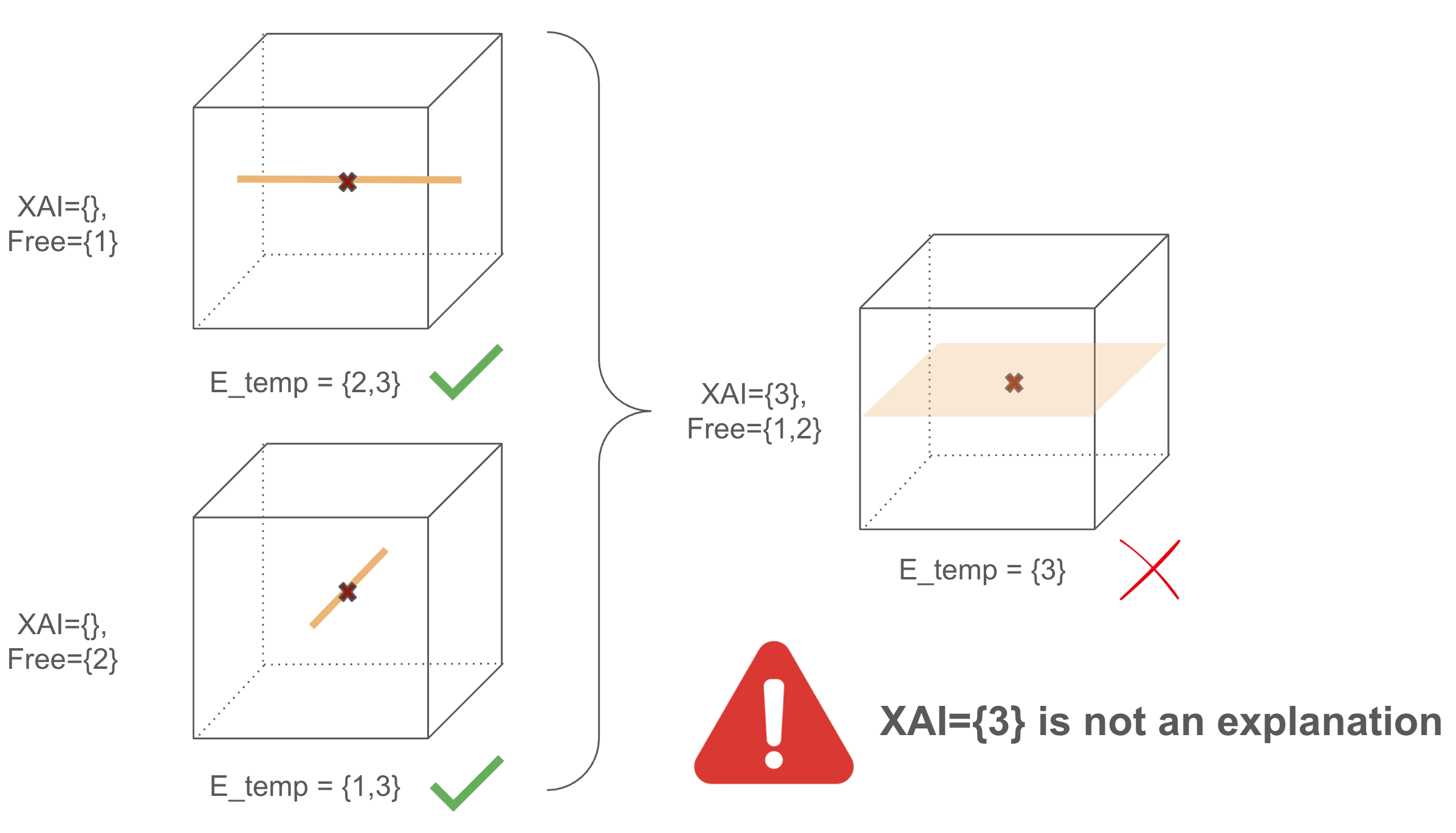}
        \text{(b) Freeing several features at once is unsound.}
    \end{minipage}
    \caption{Toy example illustrating the asymmetry between adding and freeing features.}
    \label{fig:parallel_add_free}
\end{figure}

\begin{proof}[Simultaneous Addition \ref{lemma:simultaneous_add}]
Let $\mathcal{X}$ be the current explanation candidate, and let 
$\mathcal{R} = \{r_1, \dots, r_k\}$ 
be a set of features not in $\mathcal{X}$.  
If, for every $r_i \in \mathcal{R}$, removing the single feature $r_i$ from the set 
$\mathcal{F} \setminus (\mathcal{X} \cup \{r_i\})$ 
produces a counterexample, then all features in $R$ are necessary and can be added to the explanation at once.
\end{proof}

\begin{proof}[Simultaneous freeing \ref{lemma:simultaneous_free}]
    If removing any feature from a set $\mathcal{R} \subseteq \mathcal{F} \setminus \mathcal{X}$ individually causes the explanation to fail (i.e., produces a counterexample), then all features in $\mathcal{R}$ can be added to the explanation $\mathcal{X}$ simultaneously.
\end{proof}

\begin{proof}[Batch-Certifiable Freeing \ref{lemma:batch-certificate}]
For any $i\ne c$ and $x'\in\Omega(x)$, lirpa bounds give
$f_i(x')-f_c(x')\le \overline{b}^i(x)+\sum_{j\in\mathcal{A}}\Delta_{i,j}(x')$
with $\Delta_{i,j}(x')\le c_{i,j}$. Taking the worst case over $x'$ and $i$
yields $f_i(x')-f_c(x')\le \Phi(\mathcal{A})\le 0$, precluding a label flip.

\end{proof}

-----
\subsubsection*{Proposition (Correctness of the recursive procedure)}

Let $\mathcal{A}$ be the set returned by Algorithm~\ref{alg:recursive} augmented with the final singleton refinement step that tests each remaining feature individually with the LiRPA certificate $\Phi(\cdot)$. Then:

\begin{enumerate}
  \item[(i)] (No singleton extension) For every feature $j\in \mathcal{F}\setminus\mathcal{A}$ we have
  \[
    \Phi(\mathcal{A}\cup\{j\}) > 0,
  \]
  i.e. no single feature can be added to $\mathcal{A}$ while preserving the certificate. Hence $\mathcal{A}$ is \emph{singleton-maximal} with respect to the LiRPA certificate.
  \item[(ii)] (Termination) Algorithm~\ref{alg:recursive} terminates in at most $|\mathcal{F}|$ outer iterations (and finitely many inner steps).
  \item[(iii)] (Full abstract minimality — conditional) If the inner batch solver called by Algorithm~\ref{alg:recursive} returns, for each tested budget $p$, a \emph{globally optimal} certified free set (i.e., for the current domain it finds a maximum-cardinality $\mathcal{A}_p$ satisfying $\Phi(\mathcal{A}_p)\le 0$), then the final $\mathcal{A}$ is a globally maximal certified free set: there is no $\mathcal{A}'\supsetneq\mathcal{A}$ with $\Phi(\mathcal{A}')\le 0$. In this case $\mathcal{A}$ is a true minimal abstract explanation (with respect to the chosen LiRPA relaxation).
\end{enumerate}

\paragraph{Proof.}
\textbf{(i) No singleton extension.} By construction, the algorithm performs a final singleton refinement: it tests every feature $j\in\mathcal{F}\setminus\mathcal{A}$ by evaluating the certificate on $\mathcal{A}\cup\{j\}$. The algorithm only adds $j$ to $\mathcal{A}$ if $\Phi(\mathcal{A}\cup\{j\})\le 0$. Since the refinement ends with no further additions, it follows that for every remaining $j$ we have $\Phi(\mathcal{A}\cup\{j\})>0$. This is exactly the stated property.

\textbf{(ii) Termination.} Each time the algorithm adds at least one feature to $\mathcal{A}$, the cardinality $|\mathcal{A}|$ strictly increases and cannot exceed $|\mathcal{F}|$. The outer loop therefore performs at most $|\mathcal{F}|$ successful additions. If an outer iteration yields no new features, the loop stops. Inner loops (scanning budgets $p$ or performing singleton checks) are finite since they iterate over finite sets. Hence the algorithm terminates in finite time.

\textbf{(iii) Full abstract minimality under optimal inner solver.} Suppose that for every domain tested, the inner routine (called for each $p$) returns a certified free set of maximum possible cardinality among all subsets that satisfy $\Phi(\cdot)\le 0$ on that domain. During each outer iteration the algorithm enumerates budgets $p$ (or otherwise explores the space of allowed cardinalities) and selects the largest $\mathcal{A}_p$ found; then $\mathcal{A}$ is augmented by that largest globally-feasible batch. If no nonempty globally-feasible batch exists for any tested $p$, then no superset of the current $\mathcal{A}$ can be certified (because any superset would have some cardinality $p'$ tested and the solver would have returned it). After the final singleton checks (which also use the optimal verifier on singletons), there remains no single feature that can be added. Combining these facts yields that no superset of $\mathcal{A}$ is certifiable, i.e. $\mathcal{A}$ is a globally maximal certified free set, as claimed.

\qed

\paragraph{Abstract Minimal Explanation}

\begin{proof}[Correctness of Iterative Singleton Freeing]
Let $\mathcal{F}$ be the candidate feature set and let $\mathcal{A}_0\subseteq\mathcal{F}$ be an initial free set such that the LiRPA certificate verifies $\mathcal{A}_0$ (i.e.\ $\Phi(\mathcal{A}_0)\le 0$).  
Run the Iterative Singleton Freeing procedure (Algorithm~\ref{alg:singleton_free}) with traversal order $\pi$. The algorithm returns a set $\mathcal{A}$ with the following properties:
\begin{enumerate}
  \item \textbf{(Soundness)} The final set $\mathcal{A}$ satisfies $\Phi(\mathcal{A})\le 0$ (every added singleton was certified).
  \item \textbf{(Termination)} The algorithm terminates after at most $|\mathcal{F}|-|\mathcal{A}_0|$ successful additions (hence in finite time).
  \item \textbf{(Singleton-maximality)} For every $j\in\mathcal{F}\setminus\mathcal{A}$ we have $\Phi(\mathcal{A}\cup\{j\})>0$, i.e. no remaining single feature can be certified as free.
\end{enumerate}
\end{proof}

\begin{proof}
\textbf{Soundness (invariant).} By assumption $\Phi(\mathcal{A}_0)\le 0$. The algorithm only appends a feature $i$ to the current free set after a LiRPA call returns success on $\mathcal{A}\cup\{i\}$, i.e. $\Phi(\mathcal{A}\cup\{i\})\le 0$. Since LiRPA certificates are sound, every update preserves the invariant “current $\mathcal{A}$ is certified”. Therefore the final $\mathcal{A}$ satisfies $\Phi(\mathcal{A})\le 0$.

\textbf{Termination.} Each successful iteration increases $|\mathcal{A}|$ by one and $|\mathcal{A}|\le|\mathcal{F}|$. Thus there can be at most $|\mathcal{F}|-|\mathcal{A}_0|$ successful additions. The algorithm halts when a full scan yields no addition; since scans iterate over a finite set ordered by $\pi$, the procedure terminates in finite time.

\textbf{Singleton-maximality.} Assume by contradiction that after termination there exists $j\in\mathcal{F}\setminus\mathcal{A}$ with $\Phi(\mathcal{A}\cup\{j\})\le 0$. The final scan that caused termination necessarily tested $j$ (traversal order covers all remaining indices), so the algorithm would have added $j$, contradicting termination. Hence for every $j\in\mathcal{F}\setminus\mathcal{A}$ we must have $\Phi(\mathcal{A}\cup\{j\})>0$, proving singleton-maximality.
\end{proof}

\paragraph{Worked counterexample (illustrating joint freeing).}
Consider a toy binary classifier with two input features $x_1,x_2$ and property $\mathcal{P}$: the label remains class~0 iff $f_0(x')-f_1(x')\ge 0$. Suppose the LiRPA relaxation yields conservative linear contributions such that
\[
\overline{b} + c_1 > 0,\qquad \overline{b} + c_2 > 0,\quad\text{but}\quad \overline{b} + c_1 + c_2 \le 0,
\]
where $c_i$ is the worst-case contribution of feature $i$ and $\overline{b}$ is the baseline margin. Then neither singleton $\{1\}$ nor $\{2\}$ is certifiable (each violates the certificate), but the joint set $\{1,2\}$ is certifiable. The iterative singleton procedure terminates without adding either feature, while a batch routine (or an optimal MKP solver) would free both. This demonstrates the algorithm’s limitation: it guarantees only singleton-maximality, not global maximality over multi-feature batches.

\paragraph{Complexity and practical cost.}
Let $n=|\mathcal{F}|$. In the worst case the algorithm may attempt a LiRPA call for every remaining feature on each outer iteration. If $r$ features are eventually added, the total number of LiRPA calls is bounded by
\[
(n) + (n-1) + \dots + (n-r+1) \;=\; r\cdot n - \frac{r(r-1)}{2} \;\le\; \frac{n(n+1)}{2} = \mathcal{O}(n^2).
\]
Thus worst-case LiRPA call complexity is quadratic in $n$. In practice, however, each successful addition reduces the candidate set and often many iterations terminate early; empirical behavior tends to be much closer to linear in $n$ for structured data because (i) many features are certified in early passes and (ii) LiRPA calls are highly parallelizable across features and can exploit GPU acceleration. Finally, the dominant runtime factor is the per-call cost of LiRPA (forward/backward bound propagation); therefore hybrid strategies (batch pre-filtering, prioritized traversal orders, occasional exact-solver checks on promising subsets) are useful to reduce the number of expensive LiRPA evaluations.

\subsection*{FAME for Discrete Data}\label{appendix-proof-discrete}
\newtext{
FAME, as presented, uses LiRPA, which is designed for continuous (
) domains. 
A discrete feature $j$ with admissible values in a
finite set $S_j$ can be incorporated by specifying an interval domain, which is the standard
abstraction used in LiRPA-based verification.

Consequently, FAME allows a discrete feature to vary over its admissible values. LiRPA
supports this by assigning
\[
    x'_j \in [\min S_j,\;\max S_j],
\]
or, if only a subset $S'_j \subseteq S_j$ is permitted,
\[
    x'_j \in [\min S'_j,\;\max S'_j],
\]
provided that the values form a contiguous range.

If a feature belongs to the explanation, it is fixed to its nominal value, which corresponds to
assigning the zero-width interval $[x_j, x_j]$.

Note that freeing a feature to a non-contiguous set (e.g., allowing $\{1,4\}$ but excluding
$\{2,3\}$) cannot be represented exactly, since LiRPA abstractions are convex intervals.
Extending LiRPA to arbitrary finite non-convex domains is left for future work. In
practice, such cases are rare: when categorical values have no meaningful numeric ordering,
one-hot encodings are standard, and each coordinate becomes a binary $\{0,1\}$ feature
naturally supported by interval domains.

\medskip
\noindent
Since FAME only requires sound per-feature lower and upper bounds, all its components,
including the batch certificate $\Phi(A)$ and the refinement steps, apply directly to discrete
and categorical features.

}
\section{Algorithm}\label{appendix-algo}
\newtext{
This appendix details the algorithmic procedures supporting the FAME framework and its baselines. We present four key algorithms:

\begin{itemize}
    \item \textbf{Algorithm~\ref{alg:BINARYSEARCH} (BINARYSEARCH)}: An enhanced version of the binary search traversal strategy used in Verix+. It employs a divide-and-conquer approach to identify irrelevant features, accepting a generic verification oracle (e.g., Marabou or LiRPA) as an input parameter.
    
    \item \textbf{Algorithm~\ref{alg:batch_add} (Simultaneous Add)}: An acceleration heuristic that uses adversarial attacks to quickly identify necessary features. By checking if relaxing a specific feature immediately leads to a counterexample via attacks (e.g., PGD), we can efficiently add necessary features to the explanation without expensive verification calls.
    
    \item \textbf{Algorithm~\ref{alg:singleton_free} (Iterative Singleton Freeing)}: A refinement procedure that iterates sequentially through the remaining candidate features. It utilizes LiRPA certificates to check if individual features can be safely freed, serving as a final cleanup step for features that could not be certified in batches.
    
    \item \textbf{Algorithm 5 (Recursive Abstract Batch Freeing)}: The core recursive loop of our framework. It iteratively tightens the perturbation domain using cardinality constraints (varying $m$) and invokes the greedy batch-freeing heuristic to maximize the size of the abstract explanation, concluding with a singleton refinement step.
\end{itemize}
}

\subsection{Verix+}

In this enhanced BINARYSEARCH algorithm, the solver (e.g., Marabou or Lirpa) is passed as an explicit parameter to enable the CHECK function, which performs the core verification queries.
\begin{algorithm}[H]
\caption{BINARYSEARCH($f$, $x_{\Theta}$, solver)}
\label{alg:BINARYSEARCH}
\begin{algorithmic}[1]
\Function{BINARYSEARCH}{$f, x_{\Theta}$, solver}
    \If{$|x_{\Theta}| = 1$}
        \If{CHECK($f, x_B \cup x_{\Theta}$, solver)}
            \State $x_B \gets x_B \cup x_{\Theta}$
            \State \textbf{return}
        \Else
            \State $x_A \gets x_A \cup x_{\Theta}$
            \State \textbf{return}
        \EndIf
    \EndIf
    \State $x_{\Phi}, x_{\Psi} = \text{split}(x_{\Theta}, 2)$
    \If{CHECK($f, x_B \cup x_{\Phi}$, solver)}
        \State $x_B \gets x_B \cup x_{\Phi}$
        \If{CHECK($f, x_B \cup x_{\Psi}$, solver)}
            \State $x_B \gets x_B \cup x_{\Psi}$
        \Else
            \If{$|x_{\Psi}| = 1$}
                \State $x_A \gets x_A \cup x_{\Psi}$
            \Else
                \State BINARYSEARCH($f, x_{\Psi}$, solver)
            \EndIf
        \EndIf
    \Else
        \If{$|x_{\Phi}| = 1$}
            \State $x_A \gets x_A \cup x_{\Phi}$
        \Else
            \State BINARYSEARCH($f, x_{\Phi}$, solver)
        \EndIf
    \EndIf
\EndFunction
\end{algorithmic}
\end{algorithm}

\subsection{Simultaneous Add}

\begin{algorithm}[H]
\caption{Simultaneous Add}\label{alg:batch_add}
\begin{algorithmic}[1]
\State \textbf{Input:} model $f$, input $x$, candidate set $\mathcal{F}$, current free set $\mathcal{A}$, adversarial procedure \Call{ATTACK}, property $\mathcal{P}$
\State \textbf{Initialize:} $\mathcal{E} \gets \emptyset$ \Comment{set of necessary features}
\For{$i \in \mathcal{F} \setminus \mathcal{A}$}
    \State $\mathcal{F}' \gets \mathcal{F} \setminus \{i\}$
    \If{\Call{ATTACK}{$f, \Omega(x, \mathcal{F}'), \mathcal{P}$} succeeds}
        \State $\mathcal{E} \gets \mathcal{E} \cup \{i\}$ \Comment{$i$ must remain fixed}
    \EndIf
\EndFor
\State \textbf{Return:} $\mathcal{E}$
\end{algorithmic}
\end{algorithm}

\subsection{Iterative Singleton Freeing}

\begin{algorithm}[H]
\caption{Iterative Singleton Free}\label{alg:singleton_free}
\begin{algorithmic}[1]
\State \textbf{Input:} model $f$, input $x$, candidate set $\mathcal{F}$, free set $\mathcal{A}$, certificate method \Call{LiRPA}, traversal order $\pi$, property $\mathcal{P}$
\Repeat
    \State $\texttt{found} \gets \textbf{false}$
    \For{$i \in \pi$ with $i \in \mathcal{F} \setminus \mathcal{A}$}
        \If{\Call{LiRPA}{$f, \Omega(x, \mathcal{A} \cup \{i\}), \mathcal{P}$} succeeds}
            \State $\mathcal{A} \gets \mathcal{A} \cup \{i\}$
            \State $\texttt{found} \gets \textbf{true}$
            \State \textbf{break} \Comment{restart scan from beginning of $\pi$}
        \EndIf
    \EndFor
\Until{$\texttt{found} = \textbf{false}$}
\State \textbf{Return:} $\mathcal{A}$
\end{algorithmic}
\end{algorithm}

\subsection{Recursive Simultaneous Free}

\begin{algorithm}[H]
\caption{Recursive Abstract Batch Freeing}
\begin{algorithmic}[1]
\State \textbf{Input:} model $f$, input $x$, candidate set $\mathcal{F}$
\State \textbf{Initialize:} $\mathcal{A} \gets \emptyset$ \Comment{certified free set}
\Repeat
    \State $\mathcal{A}_{best} \gets \emptyset$
    \For{$m = 1 \dots |\mathcal{F} \setminus \mathcal{A}|$}
        \State $\mathcal{A}_m \gets$ \Call{GreedyAbstractBatchFreeing}{$f, \Omega^m(x; \mathcal{A}), \mathcal{F}\setminus\mathcal{A}$}
        \If{$|\mathcal{A}_m| > |\mathcal{A}_{best}|$}
            \State $\mathcal{A}_{best} \gets \mathcal{A}_m$
        \EndIf
    \EndFor
    \State $\mathcal{A} \gets \mathcal{A} \cup \mathcal{A}_{best}$
\Until{$\mathcal{A}_{best} = \emptyset$}
\State $\mathcal{A}$ = \Call{Iterative Singleton Free}{f, x, $\mathcal{F}$, $\mathcal{A}$}\Comment{refine by testing remaining features}
\State \textbf{Return:} $\mathcal{A}$
\end{algorithmic}
\end{algorithm}
\section{Examples}\label{appendix-examples}

\subsection{Illustration of Abductive Explanation}\label{example-abductive}

\newtext{

Figure \ref{fig:toy_example_main} illustrates a 3D classification task. For the starred sample, we seek an explanation for its classification within a local cube-shaped domain. As shown in Figure \ref{fig:abductive_example_main}, fixing only feature $\mathbf{x}_2$ (i.e. freeing $\{\mathbf{x}_1,\mathbf{x}_3\}$, restricting perturbations to the orange plane) is not enough to guarantee the property, since a counterexample exists. However, fixing both $\mathbf{x}_2$ and $\mathbf{x}_3$ (orange line on free $x_1$) defines a 'safe' subdomain where the desired property holds true, since no counterexample exists in that subdomain. Therefore, $\mathcal{X}=\{\mathbf{x}_2, \mathbf{x}_3\}$ is an abductive explanation. Since neither $\{\mathbf{x}_2\}$ nor $\{\mathbf{x}_3\}$ are explanations on their own, $\{\mathbf{x}_2, \mathbf{x}_3\}$ is minimal. But it is not minimum since $\mathcal{X}=\{\mathbf{x}_1\}$ is also a minimal abductive explanation with a smaller cardinality. Two special cases are worth noting: an empty explanation (all features are irrelevant) and a full explanation (the entire input is necessary).

If all features are irrelevant, the explanation is the empty set, and no valid explanation exists. Conversely, if perturbing any feature in the input $\mathbf{x}$ changes the prediction, the entire input must be fixed, making the full feature set the explanation. 

\begin{figure}[ht!]
    \centering
    \begin{minipage}{0.44\textwidth}
        \centering
        \includegraphics[width=0.65\textwidth]{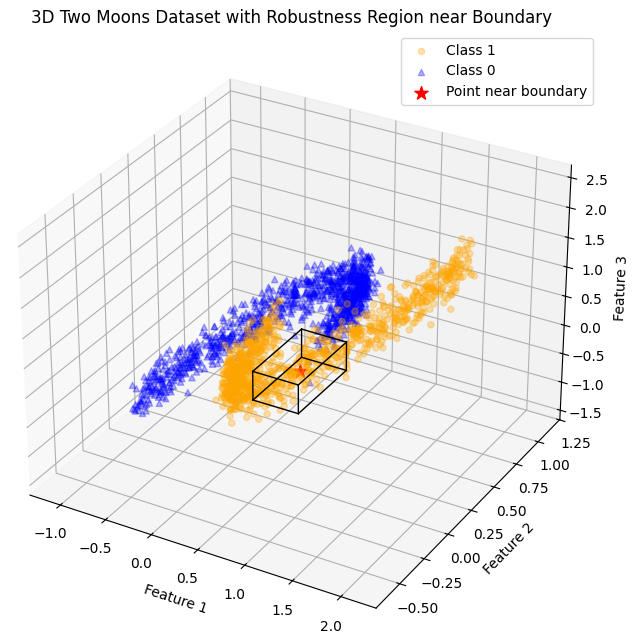}
        \caption{A 3D classification task.}
        \label{fig:toy_example_main}
    \end{minipage}
    \hfill
    \begin{minipage}{0.54\textwidth}
        \centering
        \includegraphics[width=0.9\textwidth]{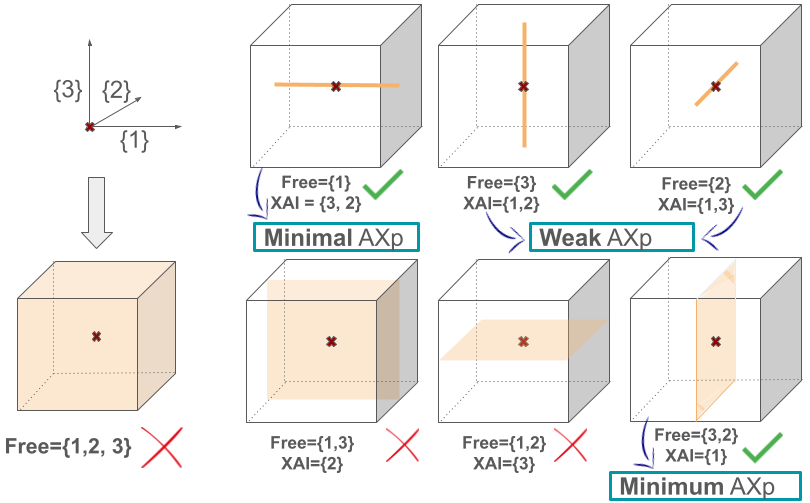}
        \caption{AXps with different properties.}
        \label{fig:abductive_example_main}
    \end{minipage}
\end{figure}
}

\subsection{Illustration of the Knapsack Formulation}\label{exp:greedy-knapsack}

\noindent
This is an example demonstrating how the greedy heuristic described in Algorithm \ref{alg:abstract-batch} works. 
Given a multi-class classification problem with three classes: 0, 1, and 2. The model correctly predicts class 0 for a given input. We want to free features from the irrelevant set $\mathcal{A}$ based on the abstract batch certificate. We have three candidate features to free: $j_1$, $j_2$, and $j_3$. The baseline budgets for the non-ground-truth classes are:

\begin{itemize}
    \item Class 1: $-\overline{b}^1 = 10$
    \item Class 2: $-\overline{b}^2 = 20$
\end{itemize}

The normalized costs for each feature are calculated as $c_{i,j} / (-\overline{b}^i)$:
\begin{table}[htbp]
  \centering
  \caption{Example of Greedy Heuristic Decision Making}
  \label{tab:greedy_heuristic_example}
  \resizebox{\linewidth}{!}{
  \begin{tabular}{l|c|c|c}
    \hline\hline
    \textbf{Feature} & \textbf{Normalized Cost for Class 1} & \textbf{Normalized Cost for Class 2} & \textbf{Maximum Normalized Cost} \\
    \textbf{($j$)} & \textbf{($c_{1,j}/(-\overline{b}^1)$)} & \textbf{($c_{2,j}/(-\overline{b}^2)$)} & \textbf{($\max_i$)} \\
    \hline
    $j_1$ & $2 / 10 = 0.2$ & $8 / 20 = 0.4$ & 0.4 \\
    $j_2$ & $7 / 10 = 0.7$ & $4 / 20 = 0.2$ & 0.7 \\
    $j_3$ & $3 / 10 = 0.3$ & $3 / 20 = 0.15$ & 0.3 \\
    \hline\hline
  \end{tabular}
  }
\end{table}

The algorithm's objective is to minimize the maximum normalized cost across all non-ground-truth classes. As shown in the table, the minimum value in the "Maximum Normalized Cost" column is 0.3, which corresponds to feature $j_3$. Therefore, the greedy heuristic selects feature $j_3$ to be added to the free set in this step, as it represents the safest choice.
\section{Experiments}\label{appendix-xps}

\subsection{Model Specification}

We evaluated our framework on standard image benchmarks including the MNIST\cite{yann2010mnist} and GTSRB\cite{stallkamp2012man} datasets. We used both fully connected and convolutional models trained in a prior state-of-the-art VERIX+\cite{wu2024better} to perform our analysis.

The MNIST dataset consists of $28 \times 28 \times 1$ grayscale handwritten images. The architectures of the fully connected and convolutional neural networks trained on this dataset are detailed in Table \ref{tab:mnist_fc} and Table \ref{tab:mnist_cnn}, respectively. These models achieved prediction accuracies of 93.76\% for the fully connected model and 96.29\% for the convolutional model.

\begin{table}[ht!]
  \centering
  \caption{Architecture of the MNIST-FC model.}
  \begin{tabular}{l|c|c|c}
    \hline\hline
    \textbf{Layer Type} & \textbf{Input Shape} & \textbf{Output Shape} & \textbf{Activation} \\
    \hline
    Flatten & $28 \times 28 \times 1$ & 784 & - \\
    Fully Connected & 784 & 10 & ReLU \\
    Fully Connected & 10 & 10 & ReLU \\
    Output & 10 & 10 & - \\
    \hline\hline
  \end{tabular}
  \label{tab:mnist_fc}
\end{table}

\begin{table}[ht!]
  \centering
  \caption{Architecture of the MNIST-CNN model.}
  \begin{tabular}{l|c|c|c}
    \hline\hline
    \textbf{Layer Type} & \textbf{Input Shape} & \textbf{Output Shape} & \textbf{Activation} \\
    \hline
    Convolution 2D & $28 \times 28 \times 1$ & $13 \times 13 \times 4$ & - \\
    Convolution 2D & $13 \times 13 \times 4$ & $6 \times 6 \times 4$ & - \\
    Flatten & $6 \times 6 \times 4$ & 144 & - \\
    Fully Connected & 144 & 20 & ReLU \\
    Output & 20 & 10 & - \\
    \hline\hline
  \end{tabular}
  \label{tab:mnist_cnn}
\end{table}

The GTSRB dataset contains colored images of traffic signs with a shape of 32×32×3 and includes 43 distinct categories. In the models used for our experiments, which were trained by the authors of \gls{verixplus}, only the 10 most frequent categories were used to mitigate potential distribution shift and obtain higher prediction accuracies. The architectures of the fully connected and convolutional models trained on GTSRB are presented in Table \ref{tab:gtsrb_fc} and Table \ref{tab:gtsrb_cnn}, respectively. These networks achieved prediction accuracies of 85.93\% and 90.32\%, respectively.

\begin{table}[ht!]
  \centering
  \caption{Architecture of the GTSRB-FC model.}
  \begin{tabular}{l|c|c|c}
    \hline\hline
    \textbf{Layer Type} & \textbf{Input Shape} & \textbf{Output Shape} & \textbf{Activation} \\
    \hline
    Flatten & $32 \times 32 \times 3$ & 3072 & - \\
    Fully Connected & 3072 & 10 & ReLU \\
    Fully Connected & 10 & 10 & ReLU \\
    Output & 10 & 10 & - \\
    \hline\hline
  \end{tabular}
  \label{tab:gtsrb_fc}
\end{table}

\begin{table}[ht!]
  \centering
  \caption{Architecture of the GTSRB-CNN model.}
  \begin{tabular}{l|c|c|c}
    \hline\hline
    \textbf{Layer Type} & \textbf{Input Shape} & \textbf{Output Shape} & \textbf{Activation} \\
    \hline
    Convolution 2D & $32 \times 32 \times 3$ & $15 \times 15 \times 4$ & - \\
    Convolution 2D & $15 \times 15 \times 4$ & $7 \times 7 \times 4$ & - \\
    Flatten & $7 \times 7 \times 4$ & 196 & - \\
    Fully Connected & 196 & 20 & ReLU \\
    Output & 20 & 10 & - \\
    \hline\hline
  \end{tabular}
  \label{tab:gtsrb_cnn}
\end{table}

\newtext{
The CIFAR-10 dataset contains colored images of common objects with a shape of $32 \times 32 \times 3$ and includes 10 distinct categories. 
The architecture of the ResNet-2B model used is detailed in Table \ref{tab:resnet_cifar}. This model (sourced from the Neural Network Verification Competition (VNN-COMP) \cite{wang2021betacrown}) is a compact residual network benchmark designed for neural network verification on CIFAR-10. Intended to help verification tools evolve beyond simple feedforward networks, this model was adversarially trained with an $L_\infty$ perturbation epsilon of $2/255$.

\begin{table}[ht!]
  \centering
  \caption{Architecture of the ResNet-2B model (CIFAR-10).}
  \begin{tabular}{l|c|c|c}
    \hline\hline
    \textbf{Layer Type} & \textbf{Input Shape} & \textbf{Output Shape} & \textbf{Activation} \\
    \hline
    Reshape & 3072 & $32 \times 32 \times 3$ & - \\
    Convolution 2D & $32 \times 32 \times 3$ & $16 \times 16 \times 8$ & ReLU \\
    Residual Block (Downsample) & $16 \times 16 \times 8$ & $8 \times 8 \times 16$ & ReLU \\
    Residual Block & $8 \times 8 \times 16$ & $8 \times 8 \times 16$ & ReLU \\
    Flatten & $8 \times 8 \times 16$ & 1024 & - \\
    Fully Connected & 1024 & 100 & ReLU \\
    Output & 100 & 10 & - \\
    \hline\hline
  \end{tabular}
  \label{tab:resnet_cifar}
\end{table}
}

\subsection{Detailed Experimental Setup}
We configured the \gls{verixplus} implementation with the following settings: binary\_search=true, logit\_ranking=true, and traversal\_order=bounds. 
To identify necessary features, we used the Fast Gradient Sign (FGS) technique for singleton attack addition, though the Projected Gradient Descent (PGD) is also available for this purpose.

\newtext{We performed a comprehensive sensitivity analysis covering: (1) Solver Choice: Table 1 shows the Greedy heuristic finds explanations nearly identical in size to the optimal MILP solver (gap $<9$ features), validating its near-optimality. (2) Cardinality Constraints: Figure 4 confirms that using the constraint (card=True) yields significantly smaller explanations. (3) Perturbation Magnitude ($\epsilon$): While we adhered to standard baseline values used by the baseline VERIX+ (e.g., 0.05 for MNIST, 0.01 for GTSRB) to ensure a direct and fair comparison, we acknowledge that explanation size is inversely related to $\epsilon$, as larger radii result in looser bounds.}





\subsection{Supplementary experimental results}\label{appendix-supp-xp}

\paragraph{PERFORMANCE WITH ITERATIVE REFINEMENT}
The three plots compare the performance of a greedy heuristic with an exact MILP solver for an iterative refinement task. The central finding across all three visualizations is that the greedy heuristic provides a strong trade-off between speed and solution quality, making it a more practical approach for large-scale problems.

\begin{figure*}[htbp]
    \centering
    \begin{minipage}[c]{0.32\textwidth}
        \centering
        \includegraphics[width=\textwidth]{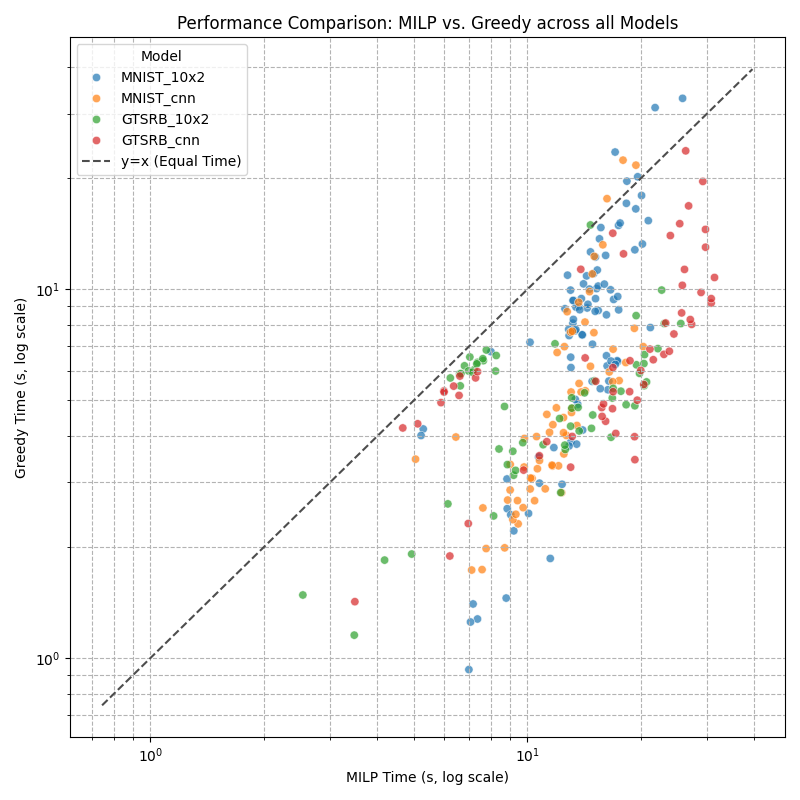}
        \label{fig:time_comparison-appendix}
    \end{minipage}
    \begin{minipage}[c]{0.32\textwidth}
        \centering
        \includegraphics[width=\textwidth]{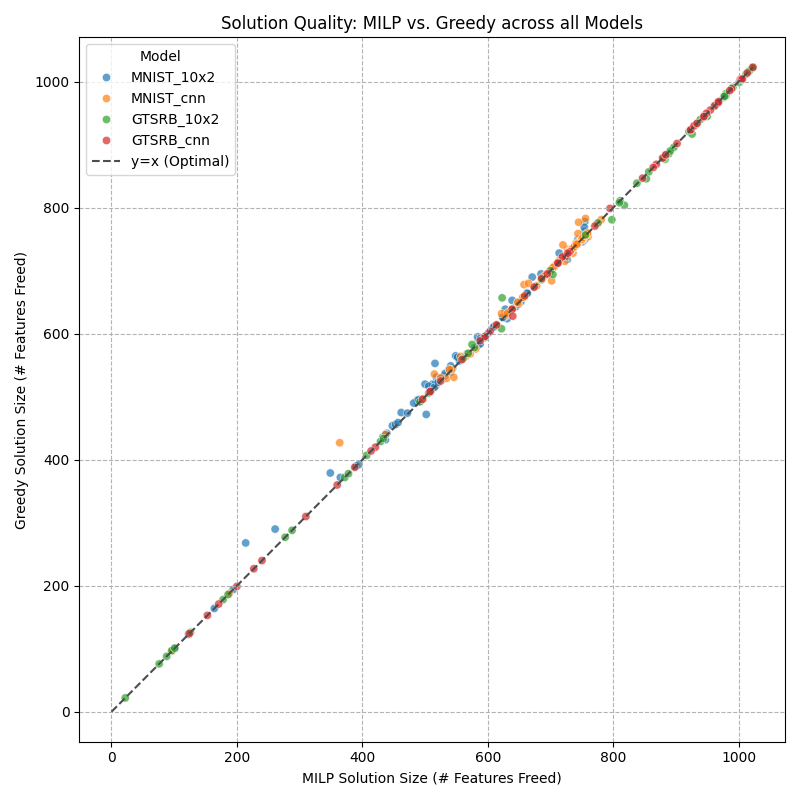}
        \label{fig:size_comparison-appendix}
    \end{minipage}
    \begin{minipage}[c]{0.32\textwidth}
        \centering
        \includegraphics[width=\textwidth]{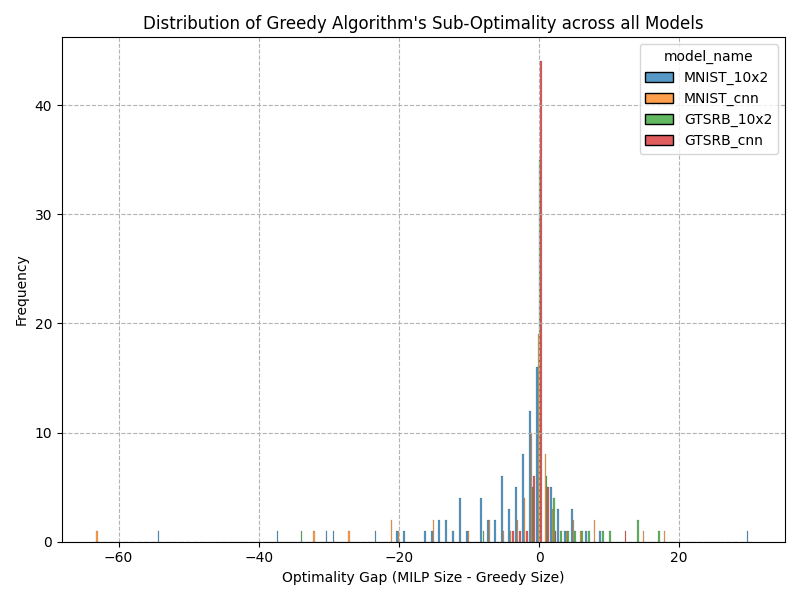}
        \label{fig:optimality_gap_histogram-appendix}
    \end{minipage}
    \caption{
        \textbf{Performance Comparison of FAME's Abstract Batch Freeing Methods.} These three plots compare the greedy heuristic against the exact MILP solver for the \textbf{iterative refinement} task for all the models. The first plot shows the runtime comparison of the two methods on a log-log scale.  
        The second plot compares the size of the freed feature set for both methods. 
        The third plot illustrates the distribution of the optimality gap (MILP size - Greedy size).
    }
    \label{fig:milp_greedy_comparison}
\end{figure*}

\textbf{Analysis of FAME's Abstract Batch Freeing}
The visualizations demonstrate that the greedy heuristic provides a strong trade-off between speed and solution quality for the iterative refinement task.

\begin{itemize}
    \item \textbf{Runtime Performance:} As shown in the first plot, the greedy algorithm is consistently faster than the MILP solver. This is evidenced by the data points for all models lying significantly below the diagonal line, confirming a substantial gain in runtime.
    \item \textbf{Solution Quality:} The second plot shows that the greedy algorithm produces solutions of comparable quality to the optimal MILP solver. The tight clustering of data points along the diagonal line for all models indicates a strong correlation between the sizes of the freed feature sets.
    \item \textbf{Optimality Gap:} The histogram of the final plot reinforces these findings by showing that the greedy heuristic frequently achieves the optimal solution, with the highest frequency of samples occurring at a gap of zero. The distribution further confirms that any sub-optimality is typically minimal.

\end{itemize}

\section{Scalability Analysis on Complex Architectures (ResNet-2B on CIFAR-10)} \label{appendix:cifar_resnet}

\newtext{To validate the scalability of FAME on architectures significantly deeper and more complex than standard benchmarks, we conducted an evaluation on the ResNet-2B model (2 residual blocks, 5 convolutional layers, 2 linear layers) trained on the CIFAR-10 dataset \cite{wang2021betacrown}. We utilized an $L_\infty$ perturbation budget of $\epsilon = 2/255$. These additional experiments were conducted on a server equipped with an NVIDIA A100 80GB GPU. 


\paragraph{Feature Definition.}
For these experiments, we define the feature set $\mathcal{F}$ at the \textbf{pixel level}. Consequently, the total number of features is $N = 32 \times 32 = 1024$. Freeing a feature in this context corresponds to simultaneously relaxing the constraints on all three color channels (RGB) for that specific pixel location.

\paragraph{Feasibility and Comparison.}

Running exact formal explanation methods (such as the complete VERIX+ pipeline with Marabou) on this architecture resulted in consistent timeouts 
or memory exhaustion, confirming that exact minimality is currently out of reach for this complexity class. In contrast, FAME successfully terminated for all processed samples. 

\paragraph{Detailed Quantitative Results by Configuration.}
To rigorously assess the contribution of each component in the FAME framework, we evaluated three configurations ($N=100$ samples). The results are summarized below:
\begin{itemize}
\item \textbf{Single-Round Abstract Freeing (Algorithm 1 only).}
    This baseline represents a static approach without domain refinement.
    \begin{itemize}
        \item \textit{Performance:} It freed an average of only \textbf{5.53 features} (pixels).
        \item \textit{Insight:} This confirms that on deep networks, the initial abstract bounds are too loose to certify meaningful batches in a single pass. A static traversal strategy would fail here.
        \item \textit{Solver Comparison:} The Greedy heuristic (5.53 features, 50.8s) performed identically to the optimal MILP solver (5.37 features, 50.8s), validating the heuristic's quality.
    \end{itemize}

\item \textbf{Recursive Abstract Refinement (Algorithm 5).}
    This configuration enables the iterative tightening of the domain $\Omega^m(x; \mathcal{A})$.
    \begin{itemize}
        \item \textit{Performance:} The average number of freed features jumped to \textbf{476.38 pixels} (approx 46\% of the image).
        \item \textit{Insight:} This dramatic increase (from $\sim$5 to $\sim$477) proves that the {adaptive abstraction} mechanism is critical. By iteratively constraining the cardinality, FAME recovers features that were previously masked by over-approximation.
        \item \textit{Solver Comparison:} Remarkably, even in this complex iterative setting, the Greedy approach (size 476.38) remained extremely close to the optimal MILP solution (size 477.76), with a negligible gap of $<0.4\%$. This strongly justifies using the faster Greedy heuristic for scalability.
        \item \textit{Runtime:} The average runtime for this intensive recursive search was approximately \textbf{1934.94 seconds} ($\sim$32 minutes).
    \end{itemize}

\item \textbf{Full Pipeline (Iteration + Singleton Refinement).}
    This represents the final output of the complete FAME pipeline, including final safety checks and singleton refinement.
    \begin{itemize}
        \item \textit{Explanation Compactness:} The pipeline successfully certified a robust explanation with an average of \textbf{240.84 freed features} (pixels) across the full dataset.
        \item \textit{Efficiency:} The breakdown confirms that FAME can navigate the search space of deep networks where exact enumerations fail, producing sound abstract explanations ($WAXp^A$) significantly faster than the timeout threshold of exact solvers.
    \end{itemize}
    
\end{itemize}

\paragraph{Discussion and Future Directions.}
While the computational cost ($\sim$32 mins) is higher than for smaller models, these results establish that the {Abstract Batch Certificate} ($\Phi$) and recursive refinement scale mathematically to residual connections without theoretical blockers. The gap between the abstract explanation size and the true minimal explanation is driven primarily by the looseness of the abstract bounds (LiRPA CROWN) on deep networks. Future work integrating tighter abstract interpretation methods (e.g., $\alpha$-CROWN) into the FAME engine will directly improve these results.

}
\section{Disclosure: Usage of LLMs}
An LLM was used solely as a writing assistant to correct grammar, fix typos, and enhance clarity.
It played no role in generating research ideas, designing the study, analyzing data, or interpreting
results; all of these tasks were carried out exclusively by the authors.
\end{document}